\documentclass[twoside]{article}
\usepackage[accepted]{aistats2018}
\RequirePackage[OT1]{fontenc}
\RequirePackage{amsthm,amsmath,amsmath,amsfonts,amssymb}
\RequirePackage[colorlinks,citecolor=blue,urlcolor=blue]{hyperref}

\RequirePackage{tablefootnote}

\RequirePackage{graphicx}

\allowdisplaybreaks[4]

\usepackage{algorithm2e}
\SetKwRepeat{Do}{do}{while}%
\SetKwInput{KwInput}{Input}
\SetKwInput{KwOutput}{Output}
\SetKwInput{KwFunction}{Function}

\usepackage{graphicx} 
\usepackage{caption}
\usepackage{subcaption}

\usepackage{titling}

\usepackage{enumitem}

\newtheorem{theorem}{Theorem}
\newtheorem{lemma}{Lemma}
\newtheorem{corollary}{Corollary}

\theoremstyle{remark}
\newtheorem{remark}{Remark}
\newtheorem{example}{Example}

\newcommand{\poly}{\mathrm{poly}}

\def\cO{\mathcal{O}}

\def\simple{\mathsf{S}}
\def\cum{\mathsf{C}}
\def\tw{\mathsf{tw}}

\newcommand{\citep}{\cite}

\begin{document}

%

%

\twocolumn[

\aistatstitle{Stochastic Zeroth-order Optimization in High Dimensions}

\aistatsauthor{ Yining Wang\;\;\;\; Simon S. Du\;\;\;\; Sivaraman Balakrishnan\;\;\;\; Aarti Singh }

\aistatsaddress{ Carnegie Mellon University } ]

\begin{abstract}
  We consider the problem of optimizing a high-dimensional convex function
  using stochastic zeroth-order queries.
  Under sparsity assumptions on the gradients or function values, 
  we present 
  two algorithms: a successive component/feature selection algorithm and a noisy mirror descent algorithm
  using Lasso gradient estimates, 
  and show that both algorithms have convergence rates that depend only logarithmically
  on the ambient  dimension of the problem.
  Empirical results confirm our theoretical findings and show that the algorithms we design outperform classical zeroth-order optimization
  methods in the high-dimensional setting.
\end{abstract}

\section{INTRODUCTION}

We consider the problem of \emph{stochastic zeroth order optimization}, where one wishes to compute the minimizer of a 
 function $f:\mathcal X\to\mathbb R$ defined on a known $d$-dimensional domain $\mathcal X\subseteq\mathbb R^d$.
 In the stochastic zeroth-order optimization setting
 the target function $f$ is unknown, and we obtain information about $f$ only through \emph{noisy} function evaluations at $T$ adaptively chosen
 points $x_1,\ldots,x_T\in\mathcal X$. At each query point $x_t$ we observe $y_t$ where, 
 \begin{equation}
 y_t = f(x_t) + \xi_t,
 \label{eq:model}
 \end{equation}
 and $\xi_t$ represents stochastic (zero-mean) noise.

 
 The (stochastic) zeroth-order optimization problem is a classical problem in optimization, machine learning, statistics, and related fields,
 and is also known as \emph{derivative-free optimization} or \emph{black-box optimization}.
 Examples include applications where gradients are difficult to evaluate and/or communicate (e.g., distributed learning and 
 parameter optimization of complicated decision processes),
 and where the function $f$ is itself unknown or inaccessible
 such as hyper-parameter tuning in machine learning and search for optimal parameters in experimental or simulation studies
 \citep{snoek2012practical,reeja2012microwave,nakamura2017design,leeds2014exploration}.

 The main focus of this paper is to understand the (convex) stochastic zeroth-order optimization problem
 in \emph{high-dimensional} scenarios,
 where the dimension of the function to be optimized ($d$) is very large and may far exceed the sample budget $T$ allowed.
 Compared to the classical stochastic \emph{first-order} optimization setting,  high dimensionality poses unique challenges in 
 the zeroth-order query model (\ref{eq:model}).
{For example, if first-order information is available (exact or noisy) and the gradient of the function $f$ is Lipschitz continuous with respect to the Euclidean distance,
the iteration complexity of the classical (stochastic) gradient descent algorithm is independent of dimension $d$;
on the other hand, the paper \cite{jamieson2012query} establishes an information-theoretic lower bound for the zeroth-order optimization problem
showing that, under the same Lipschitz continuous gradient assumptions, any zeroth-order optimization algorithm requires sample complexity growing polynomially with
the dimension $d$.
In addition, classical zeroth-order optimization algorithms such as the local averaging method \citep{flaxman2005online,agarwal2010optimal} 
have variance scaling linearly with $d$ and are not directly feasible in the high-dimensional setting.
Motivated by these observations and by real-world applications we introduce additional sparsity assumptions that enable tractable zeroth-order optimization in high dimensions.

 We propose two methods for high-dimensional zeroth-order optimization: the first method 
 uses a few noisy samples to select a small subset of ``important variables'' $S\subseteq[d]$; afterwards, existing low-dimensional zeroth-order optimization techniques are
 applied
 to obtain a minimizer of $f$ restricted to $S$.
 We also propose a different method that combines stochastic mirror descent \citep{lan2012optimal,nemirovski2009robust,duchi2010composite}
  and de-biased Lasso gradient estimates \citep{javanmard2014confidence,van2014asymptotically,zhang2014confidence}. This stochastic mirror descent based 
  method 
requires weaker assumptions and is seen 
  to outperform the variable selection based method in simulations.

  \subsection{Related work}

We conclude this section with a discussion of related works. 
The zeroth-order optimization problem and its extension to bandit convex optimization 
have been extensively studied in the machine learning and optimization literature.
The paper \cite{flaxman2005online} considers a locally smoothed surrogate of $f$ whose gradients can be unbiasedly estimated under the zeroth-order query model (\ref{eq:model})
and provides sub-linear regret bounds for the bandit convex optimization problem;
the bounds in this setting were later improved by \cite{agarwal2010optimal,shamir2013complexity,hazan2014bandit} under additional smoothness and strong convexity assumptions.
Using techniques beyond gradient-based optimization, \cite{jamieson2012query,agarwal2013stochastic,bubeck2016kernel} achieved tight dependency on the sample-budget
$T$, often obtaining worse dependency on the  dimension $d$.

There is a rich literature on sparse (high-dimensional) optimization in the classical optimization setup, 
where the gradients of the objective function $f$ can be exactly or approximately (unbiasedly) computed,
such as by sampling when the objective is a finite sum \citep{johnson2013accelerating}.
\emph{Mirror descent} \citep{nemirovski1983problem} is the classical approach for optimization with non-standard geometry
and has been applied to problems with $\ell_1$, sparsity or simplex constraints \citep{beck2003mirror,agarwal2012information,shalev2011stochastic,lan2012optimal,nemirovski2009robust,ghadimi2012optimal}.
Alternative methods such as coordinate descent \citep{shwartz2010trading} and the homotopy method \citep{xiao2013proximal} were developed to achieve faster convergence.
We remark that the first-order settings, even with noisy/stochastic gradient oracles, are very different from zeroth-order optimization because in first-order optimization
the approximate gradient estimation is usually assumed to be \emph{unbiased} with respect to the gradient of the original function, which is generally 
not possible in zeroth-order settings.

Bayesian optimization \citep{snoek2012practical} considers the same problem of optimizing an unknown function through zeroth-order query points.
Typically, in Bayesian optimization the objective function $f$ is not assumed to be convex, and the convergence rate generally scales \emph{exponentially} with problem dimension $d$ \citep{bull2011convergence,scarlett2017lower}.

The papers \cite{bandeira2012computation,bandeira2014convergence} consider the zeroth-order optimization problem 
and apply compressed sensing and sparse recovery techniques to estimate both gradients and Hessians,
and incorporate these in a trust-region algorithm.
While the high-level ideas are similar, our algorithms are gradient-based because recovery of Hessian matrices are not always computationally desirable,
especially in the high-dimensional settings.
Furthermore, no explicit finite-sample convergence rates were established in \cite{bandeira2012computation,bandeira2014convergence}. 

Lasso and $\ell_1$-penalized methods have seen great success in the fields of sparse signal recovery and high-dimensional statistical estimation \citep{tibshirani1996regression,candes2006robust,donoho2006compressed}.
Theoretical properties of Lasso such as $\ell_p$ error bounds and model selection consistency are well understood \citep{knight2000asymptotics,zhao2006model,bickel2009simultaneous,wainwright2009sharp,raskutti2011minimax}.
Recently, there has been growing interest in ``de-biasing'' the Lasso estimator in order to build component-wise confidence intervals \citep{javanmard2014confidence,zhang2014confidence,van2014asymptotically}.
We build on such de-biasing procedures in order to obtain improved rates of convergence in sparse zeroth-order optimization problems.

\section{PROBLEM SETUP} 
In this section we introduce some notation that we use throughout the paper before formally introducing 
the structural assumptions we work under and the loss metric we consider.

\subsection{Additional notation}
We write $f(n)\lesssim g(n)$ or $f(n)=\cO(g(n))$ if there exists a constant $C>0$ such that $|f(n)|\leq C|g(n)|$ for all $n\in\mathbb N$. We use $f(n) \asymp g(n)$ if
$f(n)\lesssim g(n)$ and $g(n)\lesssim f(n)$.
We also use $\widetilde\cO(\cdot)$ to suppress poly-logarithmic dependency on $n$ or $d$.

 For $1\leq p\leq\infty$, the $\ell_p$ norm of a vector $x\in\mathbb R^d$ is defined as $\|x\|_p := (\sum_{i=1}^d{|x_i|^p})^{1/p}$ for $p<\infty$, and $\|x\|_{\infty} := \max_{1\leq i\leq d}|x_i|$ for $p=\infty$.
 For two vectors $x,y\in\mathbb R^d$, the inner product $\langle\cdot,\cdot\rangle$ is defined as $\langle x,y\rangle := \sum_{i=1}^d{x_iy_i}$.
%
 A univariate random variable $X$ is \emph{sub-Gaussian} with parameter $\nu^2$ if $\mathbb E[\exp(a(X-\mathbb EX))]\leq \exp\{\nu^2a^2/2\}$ for all $a\in\mathbb R$.
 A $d$-dimensional random vector $X$ is sub-Gaussian with parameter $\nu^2$ if $\langle X-\mathbb EX,a\rangle$ is sub-Gaussian with parameter $\nu^2\|a\|_2^2$ for all $a\in\mathbb R^d$.
A random variable $X$ is \emph{sub-exponential} with parameters $(\nu,\alpha)$ if $\mathbb E[\exp(a(X-\mathbb EX))]\leq \exp\{\nu^2a^2/2\}$ for all $|a|\leq 1/\alpha$.
 
 \subsection{Assumptions and evaluation measures}

 We make the following assumptions on the target function $f:\mathcal X\to\mathbb R$ to be optimized:
 \begin{enumerate}
 \item[A1] (\emph{Unconstrained convex optimization}): We take $\mathcal X=\mathbb R^d$ and assume that 
 $f$ is convex, i.e. for all $x,x'\in\mathcal X$ and $\lambda\in[0,1]$, 
 $f(\lambda x+(1-\lambda)x')\leq \lambda f(x)+(1-\lambda)f(x')$.
 \item[A2] (\emph{Minimizer of bounded $\ell_1$-norm}): We assume there exists $x^*\in\mathcal X$ such that $f(x^*)=f^*=\inf_{x\in\mathcal X}f(x)$ and $\|x^*\|_1\leq B$;
 $x^*$ does not have to be unique.
 \item[A3] (\emph{Sparsity of gradients}): We assume that $f$ is differentiable and that there exist $H>0$, $s\ll d$ such that 
 $$
 \|\nabla f(x)\|_0\leq s, \;\;\|\nabla f(x)\|_1\leq H, \;\;\;\;\forall x\in\mathcal X,
 $$
 where $\|z\|_0$ and $\|z\|_1$ are the $\ell_0$ and $\ell_1$ vector norms;
 the support of $\nabla f(x)$ could potentially vary with $x\in\mathcal X$.
 \item[A4] (\emph{Weak sparsity of Hessians}): We assume that 
 $f$ is twice differentiable and there exists $H>0$ such that 
 $$
 \|\nabla^2 f(x)\|_1\leq H, \;\;\;\;\forall x\in\mathcal X,
 $$
 where $\|A\|_1 :=\sum_{i,j=1}^d{|A_{ij}|}$ is the entry-wise $\ell_1$ norm of matrix $A$.
 \end{enumerate}

 
 (A3) and (A4) are key assumptions in our paper, which assumes the gradients of $f$ are \emph{sparse},
 and places a weaker sparsity assumption on the Hessian matrices that constrains their $\ell_1$ norm rather than $\ell_0$ norm.
 
 We also note that, assuming $\|\nabla f(x)\|_{\infty}$ and $\|\nabla^2 f(x)\|_{\infty}$ are both bounded, both (A3) and (A4) are implied
 by the following stronger but more intuitive ``function sparsity'' assumption:
 \begin{enumerate}
 \item[A5] \emph{(Function sparsity)}: there exists $S\subseteq[d]$, $|S|\leq s$ and $f_S:\mathbb R^{|S|}\to\mathbb R$ such that $f(x)\equiv f_S(x_S)$,
 where $x_S\in\mathbb R^{|S|}$ is the restriction of $x\in\mathbb R^d$ on $S$.
 \end{enumerate}
 
 We motivate Assumptions (A3), (A4) and (A5) from both theoretical and practical perspectives.
 Theoretically, the sparsity assumption allows us to estimate the gradient 
 at a specific point using $n\ll d$ noisy zeroth-order queries.
 On the other hand, (A5) is at least approximately satisfied in many practical applications of zeroth-order optimization.
 For example, in hyper-parameter tuning problems of learning systems, it is usually the case that the performance of the system is insensitive
 to some hyper-parameters, essentially implying the sparsity of the gradients and Hessians.
Other examples include the optimization of visual stimuli so that certain types of neural responses are maximized
or optimizing experimental parameters (pressure, temperature, etc.) so that the resulting synthesized material has optimal 
quality \citep{reeja2012microwave,nakamura2017design}.
For the visual stimuli optimization example, it is well known that the hierarchical organization of the human visual system in the brain into regions such as V1, V4, LO, IT etc. is precisely based on the neural response in these regions being sensitive to specific subsets of low-level and higher-level features such as edges and curves. This 
in turn implies that the underlying function to be optimized satisfies (A5).
 Finally, we remark that similar sparsity assumptions have been considered in past work \citep{bandeira2012computation,lei2017doubly} to obtain improved rates of convergence for  optimization methods.


 \emph{Evaluation measures}: 
Let $T$ be the number of queries an algorithm $\mathcal A$ is allowed to make in the model (\ref{eq:model}),
and denote by $x_1,\ldots,x_T \in \mathcal X$ the points at which $\mathcal A$ makes queries, before producing
a final estimate $x_{T+1}$. 
 The performance of an optimization algorithm $\mathcal A$ can be measured in two ways:
 \begin{itemize}[leftmargin=*]
 \item[-] {simple regret} $R^\simple_{\mathcal A}(T) := f(x_{T+1}) - f^*$;
 \item[-] {cumulative regret} $R^{\cum}_{\mathcal A}(T) := \frac{1}{T}\sum_{t=1}^T{f(x_t)}-f^*$.
 \end{itemize}
 The simple regret $R^{\simple}_{\mathcal A}(T)$ coincides with the classical definition of optimization error and depends only on $x_{T+1}$, while the cumulative regret $R^{\cum}_{\mathcal A}(T)$
 (used extensively in online learning problems) is also affected by the quality of intermediate query points $\{x_t\}_{t=1}^T$.
 Note that both $R^\simple_{\mathcal A}(T)$ and $R^{\cum}_{\mathcal A}(T)$ are random variables, with randomness in measurement error $\{\xi_t\}_{t=1}^T$
 and the intrinsic randomness in $\mathcal A$.
 Finally, we remark that the simple regret can always be upper bounded by the cumulative regret for convex problems, since for any algorithm $\mathcal A$ that has small $R^{\cum}_{\mathcal A}(T)$, taking $x_{T+1} = \frac{1}{T}\sum_{t=1}^T{x_t}$ achieves a simple regret $R^{\simple}_{\mathcal A}(T)\leq R^{\cum}_{\mathcal A}(T)$.

\section{LASSO GRADIENT ESTIMATION}

\begin{algorithm}[t]
\KwFunction{$\textsc{GradientEstimate}(x_t,n,\delta,\lambda)$.}
Sample i.i.d.~Rademacher $z_1,\ldots,z_n\in\{1,-1\}^d$\;
Observe $\widetilde y_i=y_i/\delta$, where $y_i=f(x_t+\delta z_i)+\xi_i$\;
Let $(\widehat g_t,\widehat\mu_t)$ be the solution to Eq.~(\ref{eq:lasso})\;
\KwOutput{the Lasso gradient estimate $\widehat g_t$ and $\widehat\mu_t$.}
\vskip 0.2in
\caption{Lasso gradient estimate}
\label{alg:lasso}
\end{algorithm}

\begin{algorithm*}[t]
\KwInput{sample budget $T$, parameters $\eta,\delta,\lambda$, sparsity level $s$, minimizer norm upper bound $B$}
\textbf{Initialization}: $x_0=0$, $T'=\lfloor T/2s\rfloor$; $\widehat S_0=\emptyset$, $\widehat S_{-1}\neq\emptyset$, $t=0$; 
$\widetilde{\mathcal X}=\{x\in\mathcal X:\|x\|_1\leq B\}$\;
\While{$|\widehat S_t|<s$ and $t< s$ and $\widehat S_t\neq\widehat S_{t-1}$} {
$t\gets t+1$\;
Gradient estimation: $\widehat g_t \gets \textsc{GradientEstimate}(x_{t-1},T',\delta,\lambda)$\;
Thresholding: $\widehat S_t \gets \widehat S_{t-1}\cup\{i\in[d]: |[\widehat g_t]_i|\geq \eta\}$\;
Run finite-difference algorithm from \citep{flaxman2005online} on $f_{\widehat S_t}$ with $T'$ queries, feasible region $\widetilde{\mathcal X}$ and starting point $x_{t-1}$;
suppose the output is $x_t$\;
}
\KwOutput{$x_{T+1}=x_t$ if $|\widehat S_t|=s$ and $x_{t-1}$ otherwise.}
\vskip 0.2in
\caption{The successive component selection algorithm}
\label{alg:egs}
\end{algorithm*}

In this section we introduce the Lasso gradient estimator, which plays a central role in both our algorithms.
More specifically, for any $x_t\in\mathcal X$, the Lasso gradient estimator uses $n\ll d$ samples to estimate the unknown gradient
$g_t := \nabla f(x_t)$.
The high-level idea is to consider $n\ll d$ random samples near the point $x_t$, and to then formulate the gradient estimation problem
as a biased linear regression system.
The Lasso procedure (and its variants) can then be applied to obtain a consistent estimator under certain sparsity assumptions on $\{g_t\}_{t=1}^T$.

Fix an arbitrary $x_t\in\mathcal X$ and let $z_1,\ldots,z_n\in\{\pm 1\}^d$ be $n$ samples of i.i.d.~binary random vectors such that $\Pr[z_{ij}=1]=\Pr[z_{ij}=-1]=1/2$,
where $i\in[n]$ and $j\in[d]$.
Let $\delta >0$ be a probing parameter which will be specified later, and
$y_1=f(x_t+\delta z_1)+\xi_1, \ldots, y_n=f(x_t+\delta z_n)+\xi_n$ be the $n$ observations~\eqref{eq:model} under random perturbations (scaled by $\delta$) 
$z_1,\ldots,z_n$ of $x_t$.
Using first-order Taylor expansions with Lagrangian remainders, the normalized $\widetilde y_i:=y_i/\delta$ can be written as
\begin{align}
\widetilde y_i &= \frac{f(x_t+\delta  z_i)+\xi_i}{\delta }\nonumber\\
&= \delta ^{-1}f(x_t) + g_t^\top z_i + \frac{\delta }{2}z_i^\top H_t(\kappa_i, z_i)z_i + \delta ^{-1}\xi_i\nonumber\\
&:= \mu_t + g_t^\top z_i + \varepsilon_i,
\label{eq:gradient-model}
\end{align}
where $\mu_t =\delta ^{-1}f(x_t)$, $\varepsilon_i = \frac{\delta }{2}z_i^\top H_t(\kappa_i,z_i)z_i + \delta ^{-1}\xi_i$ and
$H_t(\kappa_i, z_i)=\nabla^2 f(x_t+\kappa_i\delta z_i)$ for some $\kappa_i\in(0,1)$.

Eq.~(\ref{eq:gradient-model}) shows that, essentially, the question of estimating $g_t=\nabla f(x_t)$ can be cast as 
a linear regression model with design $\{z_i\}_{i=1}^n$, unknown parameters $(\mu_t,g_t)\in\mathbb R^{d+1}$ and noise variables
$\{\varepsilon_i\}_{i=1}^n$ whose bias (i.e., $\mathbb E[\varepsilon_i|z_i,x_t]$) goes to 0 as $\delta \to 0$, at the expense of increasing variance.
Since $g_t$ is a sparse vector as a consequence of (A3), one can use the Lasso \citep{tibshirani1996regression} to obtain an estimate of $g_t$ and $\mu_t$:
\begin{multline}
(\widehat{g}_t, \widehat{\mu}_t) = \underset{g\in\mathbb R^d,\mu\in\mathbb R}{\text{arg min}} \frac{1}{n}\sum_{i=1}^n{(\widetilde y_i-g^\top z_i-\mu)^2} +  \\
~~~~~~~~~~~~~~~~~~~~~~~~~~\lambda\|g\|_1 + \lambda|\mu|,
\label{eq:lasso}
\end{multline}
where $\lambda>0$ is a regularization parameter that will be specified later. A pseudocode description of the Lasso gradient estimator is given in Algorithm \ref{alg:lasso}.
The following lemma shows that with a carefully chosen $\lambda$, $\widehat g_t$ is a good estimate of $g_t$ in both $\ell_{\infty}$ and $\ell_1$ norms.
\begin{lemma}
Suppose (A1) through (A4) hold.
Suppose also that $n=\Omega(s^2\log d)$, $n\leq d$ and $\lambda\asymp \delta ^{-1}\sigma\sqrt{\log d/n} + \delta H$.
Then with probability $1-\cO(d^{-2})$
$$
\max\{|\widehat\mu_t-\mu_t|, \|\widehat g_t-g_t\|_{\infty}\} \lesssim \frac{\sigma}{\delta}\sqrt{\frac{\log d}{n}} + \delta H.
$$
Furthermore, with probability $1-\cO(d^{-2})$ it holds that $\|\widehat g_t-g_t\|_1\leq 2s\|\widehat g_t-g_t\|_{\infty}$.
\label{lem:lasso}
\end{lemma}
Lemma \ref{lem:lasso} follows by the standard $\ell_1$ and $\ell_{\infty}$ error bound analyses of the Lasso estimator \citep{bickel2009simultaneous,lounici2008supnorm}.
However, our model has a subtle difference from the standard high-dimensional regression model in that $\mathbb E[\varepsilon_i|z_i,x_t]$ are not exactly zero.
and we provide a detailed proof in the Appendix.
\begin{remark}
The penalization of $\mu$ in Eq.~(\ref{eq:lasso}) is in general unnecessary as it is a single component;
however, we decide to keep this penalization term to simplify our analysis.
Neither the estimation error nor the selection of the tuning parameter $\lambda$ depend on knowledge of $\mu_t$.
\end{remark}

\begin{remark}
Lemma \ref{lem:lasso} reveals an interesting bias-variance tradeoff controlled by the ``probing'' parameter $\delta >0$.
When $\delta $ is close to 0, the bias (reflected by $\mathbb E[\varepsilon_i|z_i,x_t]$) resulting from the second-order Lagrangian remainder term $\frac{\delta }{2}z_i^\top H_t(\kappa_i,z_i)z_i$ is small; however, the variance of $\widehat g_t$ is large because the variance of the ``stochastic'' noise term $\xi_i/\delta $ increases as $\delta \to 0$;
on the other hand, for large $\delta $ the stochastic variance is reduced but the bias from first-order approximation of $f(x_t)$ increases.
\end{remark}

\section{COMPONENT SELECTION}\label{sec:model-selection}

Given the estimation error bound of the Lasso gradient estimator and the stronger ``function sparsity'' assumption (A5),
 our first attempt is to use $\widehat g_t$ to select a few ``relevant'' components $\widehat S\subseteq[d]$, $|\widehat S|\ll d$
and perform classical low-dimensional zeroth-order optimization restricted to $\widehat S$.
%
The following corollary shows that, the components in $S$ whose gradients have large absolute values can be detected by a thresholding Lasso estimator:
\begin{corollary}
Suppose the conditions in Lemma \ref{lem:lasso} hold and let $\eta=\omega\lambda$ depending on some sufficiently large constant $\omega>1$.
Let $\widehat S(\eta) := \{i\in[d]: |[\widehat g_t]_i|>\eta\}$ be the selected components by thresholding the Lasso estimate $\widehat g_t$. 
Then with probability $1-\cO(d^{-2})$
$$
\left\{i\in S: |[\nabla f(x_t)]_i|>2\eta\right\} \subseteq \widehat S(\eta) \subseteq S.
$$
\label{cor:model-selection}
\end{corollary}
Corollary \ref{cor:model-selection} can be proved by directly applying the $\|\widehat g_t-g_t\|_{\infty}$ bound in Lemma \ref{lem:lasso}.
It shows that with threshold $\eta=\omega\lambda$ depending on some sufficiently large constant $\omega>1$,
the thresholding estimator $\widehat S(\eta)$ with high probability will not include components that do not belong to $S$ (i.e., no false positives).
On the other hand, all components in $S$ that have a sufficiently large partial derivative (at $x_t$) will be detected by $\widehat S(\eta)$.

Algorithm \ref{alg:egs} describes the pseudo-code of a ``successive'' component selection algorithm inspired by the above observations.
The following theorem provides a convergence analysis for Algorithm \ref{alg:egs}:
\begin{theorem}
Suppose (A1) through (A5) hold.
Suppose also that $T=\Omega(s^3\log d)$ and $T\leq d$.
Let parameters $\delta,\lambda,\eta$ be set as $\delta\asymp \left(\frac{\sigma^2s\log d}{H^2T}\right)^{1/4}$,
$\lambda\asymp \frac{\sigma}{\delta}\sqrt{\frac{s\log d}{T}}+\delta H$ and $\eta = \omega\lambda$
depending on some sufficiently large constant $\omega>1$. Then with probability at least 0.9
\begin{equation}
R_{\mathcal A}^{\simple}(T) \lesssim B\left(\frac{\sigma^2 H^2s\log d}{T}\right)^{1/4} + \widetilde{\cO}(T^{-1/3}),
\label{eq:model-selection-rate}
\end{equation}
\label{thm:model-selection-rate}
\end{theorem}
The proof of Theorem \ref{thm:model-selection-rate} is essentially a repeated application of Corollary \ref{cor:model-selection}, which we defer to the appendix.
\begin{remark}
In the $\widetilde{\cO}(\cdot)$ notation in Eq.~(\ref{eq:model-selection-rate}) we suppress polynomial dependency on $\sigma,s,H,B$ and $\log d$.
The $\lesssim$ notation does not suppress dependency on any problem dependent constants.
\end{remark}
\begin{remark}
The choices of $\lambda$ and $\delta$ differ by factors depending on $s$ from the choices suggested by Lemma~\ref{lem:lasso}. This is due to the fact that we 
divide the sample budget over $s$ rounds of component selection. 
\end{remark}


\begin{remark}
Theorem \ref{thm:model-selection-rate} only upper bounds the simple regret of the successive component selection algorithm $\mathcal A$.
However, it is clear that Algorithm \ref{alg:egs} cannot achieve consistent cumulative regret bounds,
because the gradient estimation step already consumes a constant fraction of sample points (up to $\cO(s)$ factors).
\end{remark}

\begin{remark}
The failure probability of Theorem \ref{thm:model-selection-rate} is at a constant level and does not go to 0 as $d$ or $T$ go to infinity.
This is a consequence of the fact that the $T^{-1/3}$ regret bound of the paper \citep{flaxman2005online} for low-dimensional zeroth-order optimization 
only holds in expectation. To the best of our knowledge, exponential tail bounds remain an open question \citep{shamir2013complexity}.
\end{remark}

\section{MIRROR DESCENT}\label{sec:mirror-descent}

\begin{algorithm*}[t]
\KwInput{minimizer norm $B$, sample budget $T$, gradient estimate budget $n$, potential $\psi$, parameters $\eta,\delta,\lambda$.}
\textbf{Initialization}: $x_0=0$, $T':=\lfloor T/2n\rfloor$, $\widetilde{\mathcal X} := \{x: \|x\|_1\leq B\}$\; 
\For{$t=0,\ldots,T'-1$} {
	Lasso gradient estimation: $(\widehat g_t,\widehat\mu_t) \gets \textsc{GradientEstimate}(x_t, 2n, \delta, \lambda)$\;
	De-biasing: $\widetilde g_t\gets \widehat g_t+\frac{1}{n}Z_t^\top(\widetilde Y_t-Z_t\widehat g_t-\widehat\mu_t\cdot 1_{n})$\; 
	MD update: $x_{t+1} \gets \arg\min_{x\in\widetilde{\mathcal X}}\{\eta\widetilde g_t^\top(x-x_t) + \Delta_\psi(x,x_t)\}$\;
}
\vskip 0.2in
\caption{First-order mirror descent with estimated gradients}
\label{alg:main}
\end{algorithm*}

Another possibility of applying the Lasso gradient estimator $\widehat g_t$ for optimizing $f$ is to consider classical or sparse first-order methods (e.g., SGD or mirror descent),
with the true gradients $g_t=\nabla f(x_t)$ at each iteration replaced by their estimates $\widehat g_t$.
However, directly plugging in the Lasso estimator leads to poor convergence properties due to the inherent estimation bias in $\widehat g_t$.
To overcome such difficulties, we consider the recent work on \emph{de-biased} Lasso estimators \citep{javanmard2014confidence,van2014asymptotically,zhang2014confidence}
and apply stochastic mirror descent \citep{nemirovski1983problem} to handle the entrywise error introduced by the de-biasing estimators.

\subsection{De-biased Lasso estimation}

The de-biased Lasso estimator was introduced in \citep{zhang2014confidence} and generalized in \citep{javanmard2014confidence,van2014asymptotically}
to reduce bias of the Lasso estimator for the purpose of constructing confidence intervals for low-dimensional model components.
In our application, the bias-reduced gradient estimate allows stochastic noise to concentrate across epochs and leads to improved convergence rates.

Let $\widetilde Y_t=(\widetilde y_1,\ldots,\widetilde y_n)\in\mathbb R^n$ and $Z_t=(z_1,\ldots,z_n)\in\mathbb R^{n\times d}$ be the vector forms 
of $\{\widetilde y_i\}_{i=1}^n$ and $\{z_i\}_{i=1}^n$.
Since the design points $z_i$ are i.i.d.~Rademacher variables, the de-biased gradient estimator $\widetilde g_t$ takes a particularly simple form:
\begin{flalign*}
\textbf{The de-biased Lasso}:&&
\end{flalign*}
\begin{equation}
\widetilde g_t := \widehat g_t + \frac{1}{n}Z_t^\top(\widetilde Y_t-Z_t\widehat g_t-\widehat\mu_t\cdot 1_n).
\label{eq:debias}
\end{equation}
Here $(\widehat g_t,\widehat\mu_t)$ is the Lasso estimator defined in Eq.~(\ref{eq:lasso}) and $1_n = (1,\ldots,1)\in\mathbb R^n$ is the $n$-dimensional vector of all ones.

\begin{lemma}
Suppose $n=\Omega(s^2\log d)$.
With probability $1-\cO(d^{-2})$ it holds that
$$
\widetilde g_t=g_t + \zeta_t + \gamma_t;
$$
where $\zeta_t$ is a $d$-dimensional random vector such that, for any $a\in\mathbb R^d$, $\langle\zeta_t,a\rangle$ conditioned on $x_t$
is a centered sub-exponential random variable with parameters $\nu=\sqrt{n/2}\cdot\alpha$ and $\alpha\lesssim \sigma\|a\|_2/\delta{n}$;
and $\gamma_t$ is a $d$-dimensional vector that satisfies
$$
\|\gamma_t\|_{\infty} \lesssim H\delta + \frac{\sigma s\log d}{\delta n} \;\;\;\;\;\;\text{almost surely.}
$$
\label{lem:debias}
\end{lemma}

Comparing Lemma \ref{lem:debias} with the error bound obtained for the Lasso estimator $\widehat g_t$ in Lemma \ref{lem:lasso}, it is clear that
the entry-wise bias (i.e., $\|\gamma_t\|_{\infty}$) is reduced from $\cO(\delta H+\sqrt{\log d/\delta n})$ 
to $\cO(\delta H + s\log d/\delta n)$. 
Such de-biasing is at the cost of inflated stochastic error $\zeta_t$, which means that unlike $\widehat g_t$, $\widetilde g_t$ is not a good estimator of $g_t$
in the $\ell_1$ or $\ell_2$ norm.

\subsection{Bregman divergence and stochastic mirror descent}

Mirror descent (MD) \citep{nemirovski1983problem} is a classical method in optimization when smoothness and the domain geometry
are measured in (possibly) non-Euclidean metrics. 
The MD algorithm was applied to stochastic optimization with noisy first-order oracles in the papers \citep{nemirovski2009robust,agarwal2012information}
and was also studied in the work \citep{lan2012optimal} for strongly smooth composite functions with accelerated rates,
and in the works \citep{ghadimi2012optimal,ghadimi2013optimal} for strongly convex composite functions.

Let $\psi:\mathcal X\to\mathbb R$ be a continuously differentiable, strictly convex function.
The \emph{Bregman divergence} $\Delta_\psi: \mathcal X\times\mathcal X\to\mathbb R$ is defined as
\begin{equation}
\Delta_\psi(x,y) := \psi(y)-\psi(x)-\langle\nabla\psi(x),y-x\rangle.
\label{eq:breg-sc}
\end{equation}
Let $\|\cdot\|_\psi$ be a norm 
and $\|\cdot\|_{\psi^*}$ be its dual norm, defined as $\|z\|_{\psi^*} := \sup\{z^\top x: \|x\|_{\psi}\leq 1\}$.
One important class of Bregman divergences is those that are $\kappa$-\emph{strongly convex} with respect to the chosen norm, 
i.e. they satisfy
$\Delta_\psi(x,y) \geq \frac{\kappa}{2}\|x-y\|_{\psi}^2$.
Many choices of $\psi$ lead to a strongly convex Bregman divergence. In this paper we consider the $\ell_a$ norm as choice of $\psi$:
$
\psi_a(x) := \frac{1}{2(a-1)}\|x\|_a^2
$ 
for $1<a\leq 2$.
It was proved in \citep{agarwal2012information,srebro2011universality} that $\psi_a$ leads to a valid Bregman divergence that satisfies 1-strong convexity with respect to $\|\cdot\|_a$.
%
With this setup, the MD method iteratively computes
\begin{equation*}
x_{t+1} := \underset{x\in\widetilde{\mathcal X}}{\text{arg min}} \left\{\eta_t \nabla f(x_t)^\top(x-x_t) + \Delta_\psi(x,x_t)\right\},
\label{eq:md}
\end{equation*}
where $\{\eta_t\}_{t=1}^T$ is a sequence of step sizes
and $\widetilde{\mathcal X}\subseteq\mathcal X$ is a subset of the domain $\mathcal X$ of $f$.
In our problem where $\nabla f(x_t)$ is not accessible, the de-biased Lasso gradient estimate $\widetilde g_t$
is used to replace the exact gradient $\nabla f(x_t)$.
A pseudo-code description of our method is given in Algorithm \ref{alg:main}.

\subsection{Rates of convergence}

We present the following convergence rate for Algorithm \ref{alg:main}, which is proved in the appendix:
\begin{theorem}
Suppose (A1) through (A4) hold.
Suppose also that $T=\Omega(s^3\log^2d + s(1+H)^2(1+B^4H^4\log^2d))$, $T\leq d$ and that we
choose the parameters $n := \left\lfloor{(1+H)\sqrt{sT}}\right\rfloor$, 
$\eta: = B\sqrt{\frac{n\log d}{T}}$,
and $\delta:=\sqrt{s\log d/n}$.
Then with probability $1-\cO(d^{-1})$
\begin{equation*}
R_{\mathcal A}^{\cum}(T) \lesssim \xi_{\sigma,s}B\sqrt{\log d}\left[\frac{(1+H)^2s}{T}\right]^{1/4} +\widetilde \cO(T^{-1/2}).
\end{equation*}
where $\xi_{\sigma,s}=1+\sigma+\sigma^2/s$.
\label{thm:main}
\end{theorem}


Theorem \ref{thm:main} shows that Algorithm \ref{alg:main} has similar convergence rate as the successive component selection algorithm,
but operates under weaker conditions (i.e., without the function sparsity assumption (A5)).
There are also two additional differences between results in Theorems \ref{thm:model-selection-rate} and \ref{thm:main}.
First, Theorem \ref{thm:main} upper bounds the cumulative regret $R_{\mathcal A}^{\cum}(T)$, while the error bound in Theorem \ref{thm:model-selection-rate}
only applies to the simple regret $R_{\mathcal A}^{\simple}(T)$.
Furthermore, the error bound in Theorem \ref{thm:main} holds with high-probability ($1-\cO(d^{-1})$),
while the results in Theorem \ref{thm:model-selection-rate} only hold with constant probability.

\subsection{Improved rates with Hessian smoothness}

We show an extension of our algorithm that greatly improves the convergence rate 
under additional smoothness conditions on $\nabla^2 f$, with a small loss in computational efficiency.
Formally, we assume:
\begin{enumerate}
\item[A6] (\emph{Hessian smoothness}). There exists $L>0$ such that for all $x,x'\in\mathcal X$,
$$
\|\nabla^2 f(x)-\nabla^2 f(x')\|_1 \leq L\|x-x'\|_\infty
$$
\end{enumerate}
Recall that $\|A\|_1 = \sum_{i,j}|A_{ij}|$ denotes the entry-wise $\ell_1$ norm of a matrix $A$.

If $f$ is three-times differentiable, then (A6) is implied by the condition that $\|\nabla^3 f(x)\|_1\leq L$ for all $x\in\mathcal X$,
where $\|A\|_1 := \sum_{i,j,k}|A_{ijk}|$ is the entry-wise $\ell_1$ norm of a third order tensor.
However, (A6) in general does not require third-order differentiability of $f$.

Recall the de-biased Lasso gradient estimator $\widetilde g_t(\delta)$ in Eqs.~(\ref{eq:lasso},\ref{eq:debias})
corresponding to a probing step size of $\delta$. 
Under the additional condition (A6), the analysis in Lemma \ref{lem:debias} can be strengthened as below:
\begin{lemma}
Suppose (A1) through (A4) and (A6) hold.
Suppose also that $n=\Omega(s^2\log d)$, $n\leq d$ and $\lambda\asymp \delta^{-1}\sigma\sqrt{\log d/n}+\delta H$.
Then with probability $1-\cO(d^{-2})$
$$
\widetilde g_t(\delta) = g_t + \frac{\delta}{2}\mathbb E\left[(z^\top H_t z)z\right] + \widetilde\zeta_t(\delta) + \widetilde\beta_t(\delta) + \widetilde\gamma_t(\delta),
$$
where $g_t=\nabla f(x_t)$, $H_t=\nabla^2 f(x_t)$;
for any $a\in\mathbb R^d$, 
$\langle\widetilde\zeta_t(\delta),a\rangle$ conditioned on $x_t$ is a centered $d$-dimensional sub-exponential random variable
with parameters $\nu^2=\sqrt{n/2}\cdot\alpha$ and $\alpha\lesssim \sigma\|a\|_2/\delta{n}$;
$\langle\widetilde\beta_t(\delta),a\rangle$ conditioned on $x_t$ is a centered $d$-dimensional sub-Gaussian random variable
with parameter $\nu\lesssim \delta H\|a\|_1/\sqrt{n}$;
$\gamma_t(\delta)$ is a $d$-dimensional vector that satisfies
$$
\|\widetilde\gamma_t(\delta)\|_{\infty} \lesssim L\delta^2 + \frac{\sigma s\log d}{n\delta} + s\delta H\sqrt{\frac{\log d}{n}}.
$$
\label{lem:hs-debias}
\end{lemma}

Note that $\widetilde\zeta_t(\delta)$ and $\widetilde\beta_t(\delta)$ might be correlated conditioned on $x_t$.
Comparing Lemma \ref{lem:hs-debias} with Lemma \ref{lem:debias}, we observe that the bias term $\widetilde\gamma_t(\delta)$ is significantly smaller ($\cO(\delta^2)$ instead of $\cO(\delta)$);
while the second term $\frac{\delta}{2}\mathbb E[(z^\top H_t z) z]$ is still a bias term with non-zero mean, it only depends on $\delta$ and can be easily removed.
This motivates the following definition of a ``twice de-biased'' gradient estimator:
\begin{flalign*}
\textbf{The twice de-biased estimator:}&&
\end{flalign*}
\begin{equation}
\widetilde g_t^{\tw} := 2\widetilde g_t(\delta/2) - \widetilde g_t(\delta).
\label{eq:hs-debias}
\end{equation}

\begin{corollary}
Suppose the conditions in Lemma \ref{lem:hs-debias} are satisfied.
Then with probability $1-\cO(d^{-2})$, 
$$\widetilde g_t^{\tw}-g_t = \widetilde\zeta_t+\widetilde\beta_t+\widetilde\gamma_t,$$
 where
$\widetilde\zeta_t=2\widetilde\zeta_t(\delta/2)-\widetilde\zeta_t(\delta)$, $\widetilde\beta_t=2\widetilde\beta_t(\delta/2)-\widetilde\beta_t(\delta)$ and $\widetilde\gamma_t=\widetilde\gamma_t(\delta/2)-\widetilde\gamma_t(\delta)$.
\label{eq:hs-debias}
\label{cor:hs-debias}
\end{corollary}

\begin{figure}[t]
	\centering
	\begin{subfigure}[t]{0.23\textwidth}
		\includegraphics[width=\textwidth]{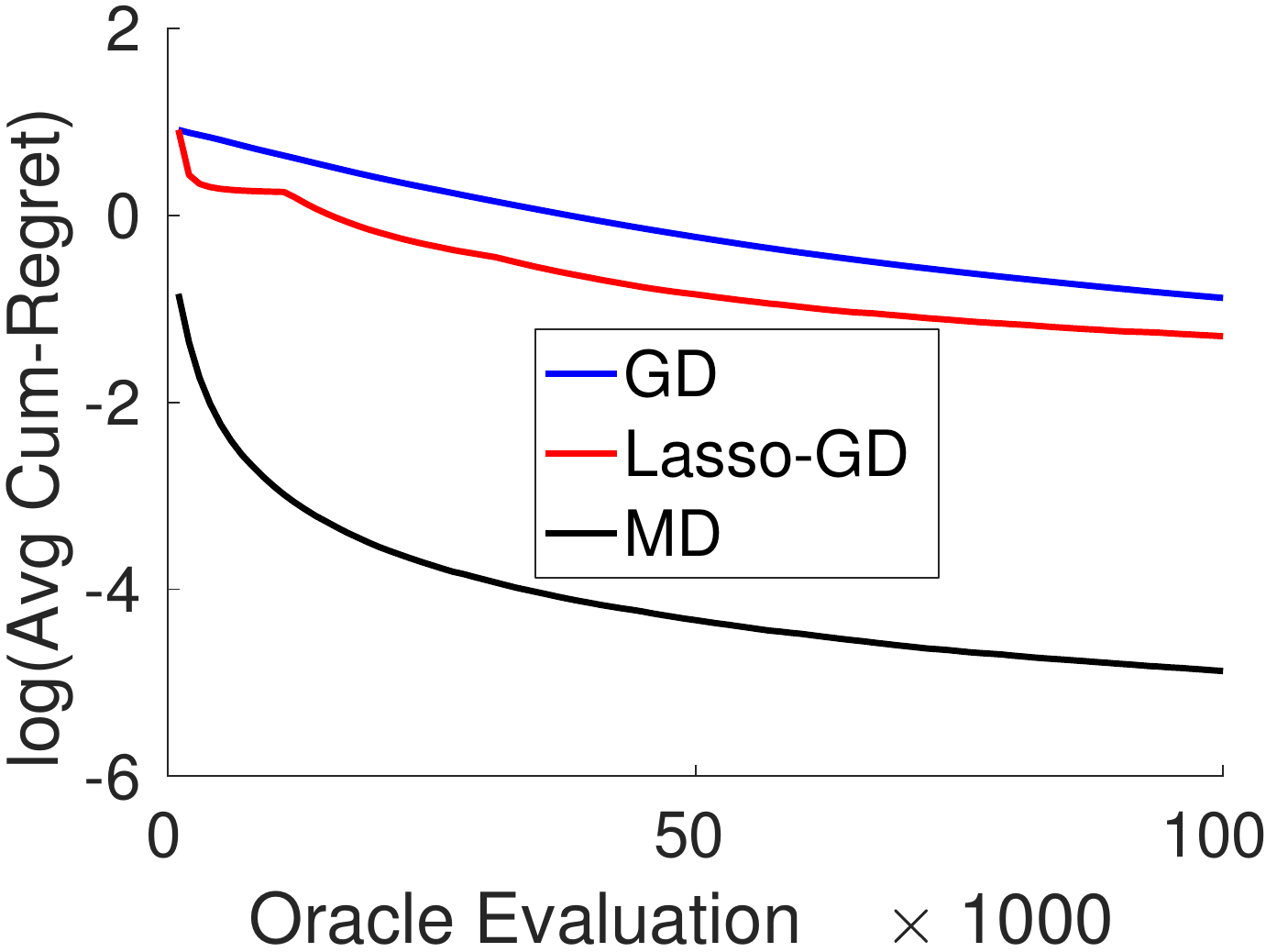}
		\caption{$s=10,d=100$}\label{fig:s10_d100}
	\end{subfigure}	
	\begin{subfigure}[t]{0.23\textwidth}
		\includegraphics[width=\textwidth]{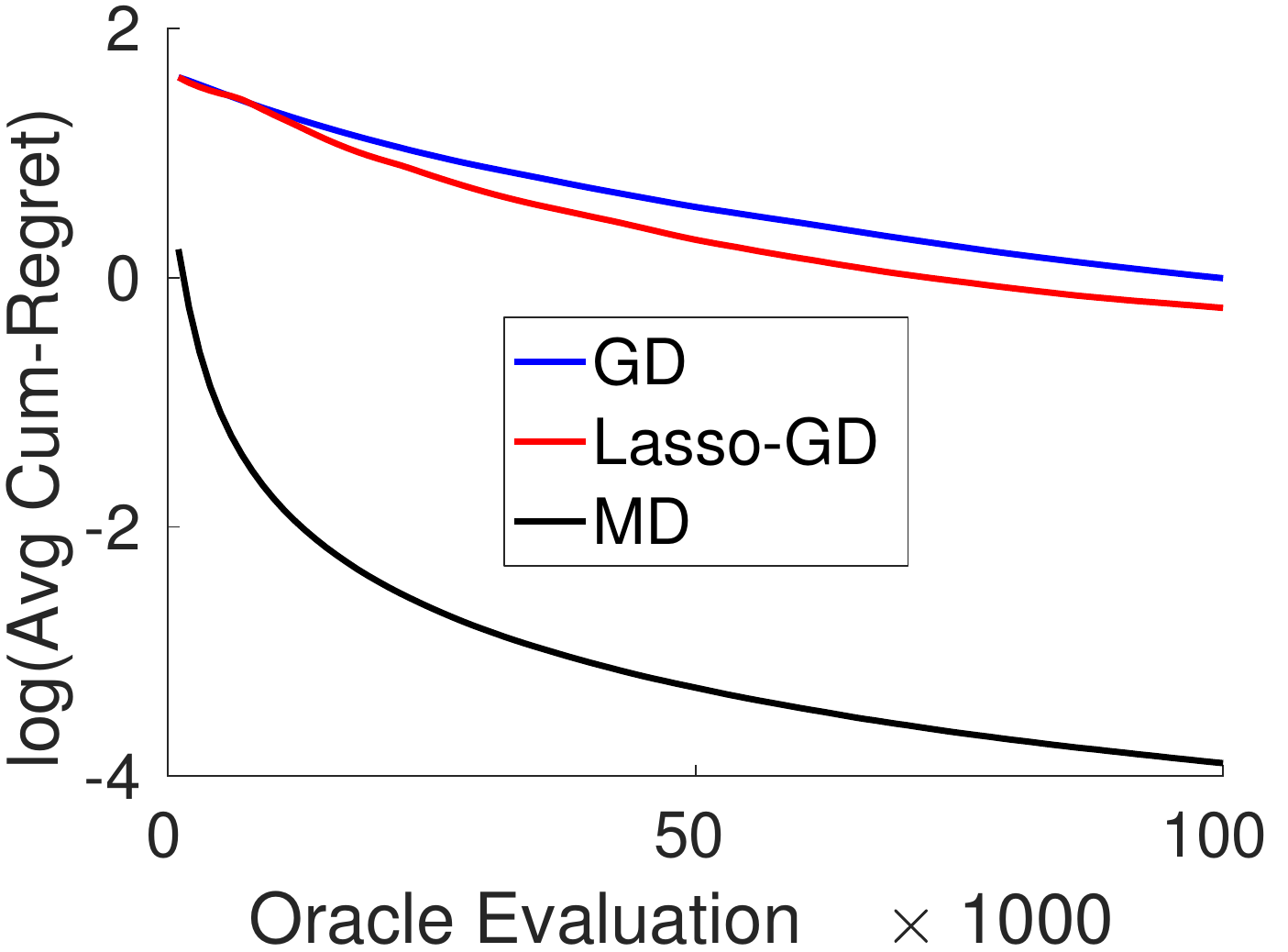}
		\caption{$s=20,d=100$}\label{fig:d20_d100}
	\end{subfigure}	
\caption{Sparse quadratic optimization with identity quadratic term.}\label{fig:identity}
\end{figure}
\begin{figure}[t]
\centering
	\begin{subfigure}[t]{0.23\textwidth}
		\includegraphics[width=\textwidth]{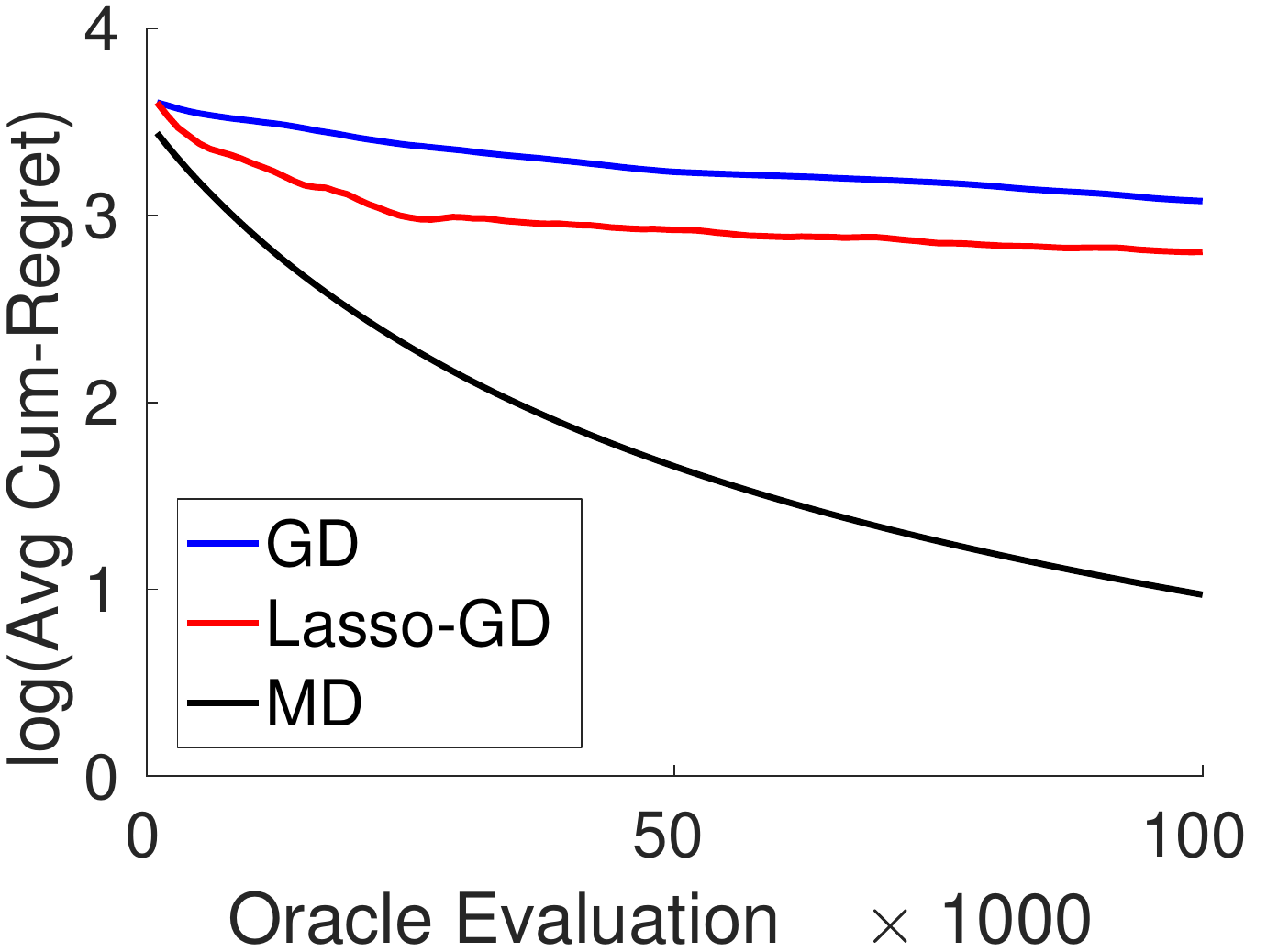}
		\caption{decay rate $=1.5$}\label{fig:decay15}
	\end{subfigure}	
	\begin{subfigure}[t]{0.23\textwidth}
		\includegraphics[width=\textwidth]{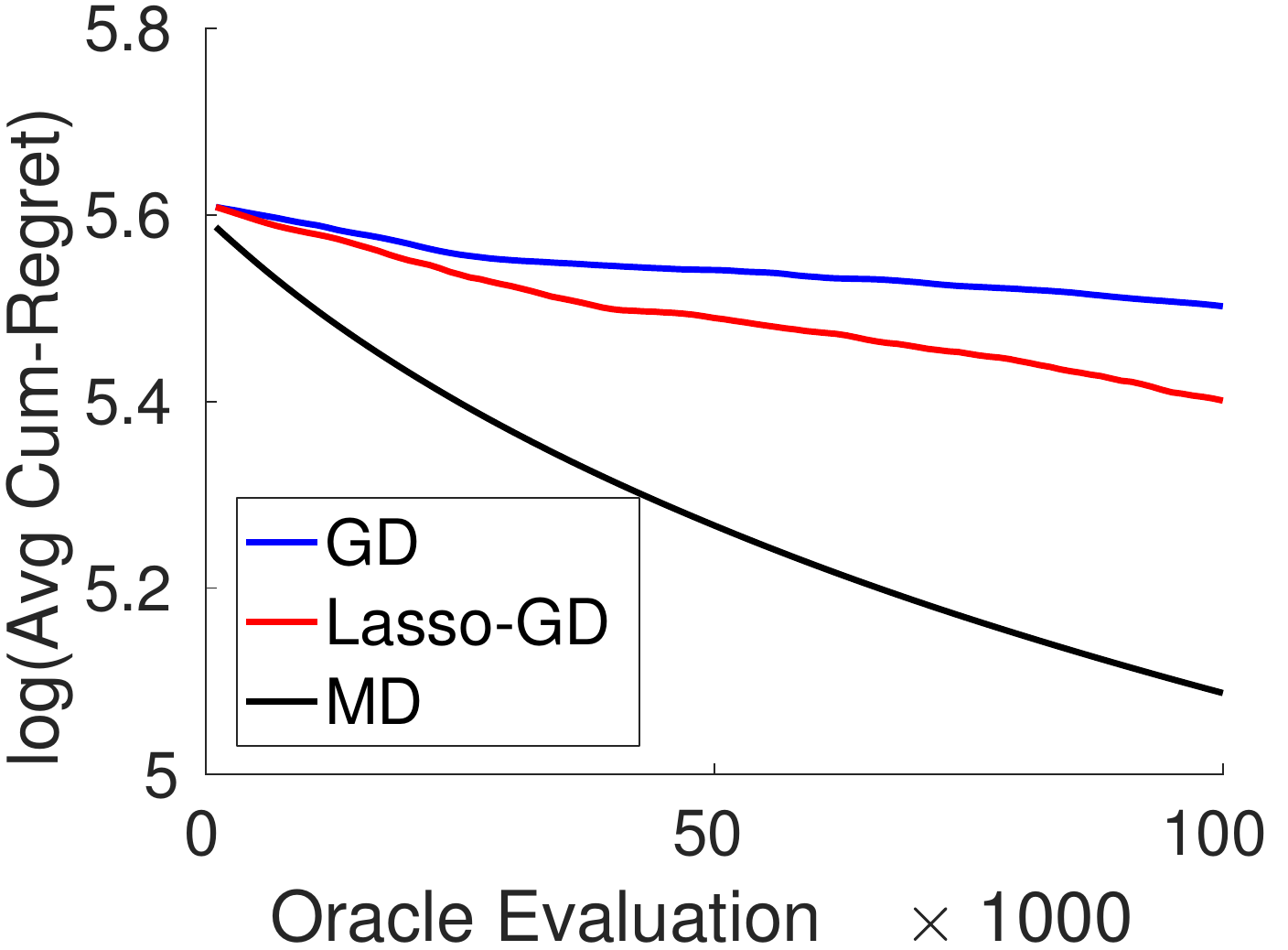}
		\caption{decay rate $=3$}\label{decay3}
	\end{subfigure}	
\caption{Sparse quadratic optimization with polynomial decay of eigenvalues.}\label{fig:polydecay}
\end{figure}
\begin{figure}[t]
	\centering
	\begin{subfigure}[t]{0.23\textwidth}
		\includegraphics[width=\textwidth]{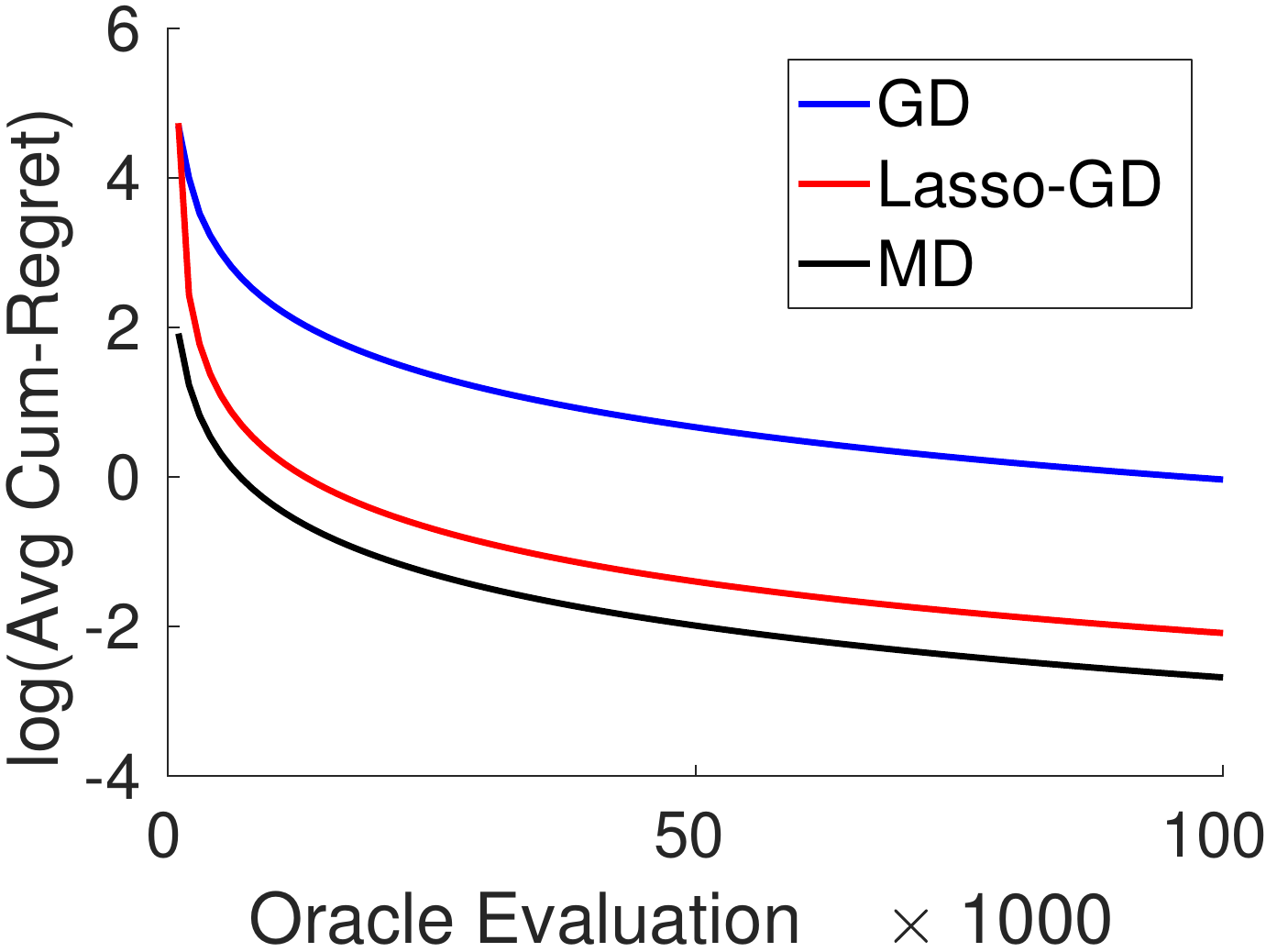}
		\caption{$s=10,d=100$}\label{fig:fourth_s10_d100}
	\end{subfigure}	
	\begin{subfigure}[t]{0.23\textwidth}
		\includegraphics[width=\textwidth]{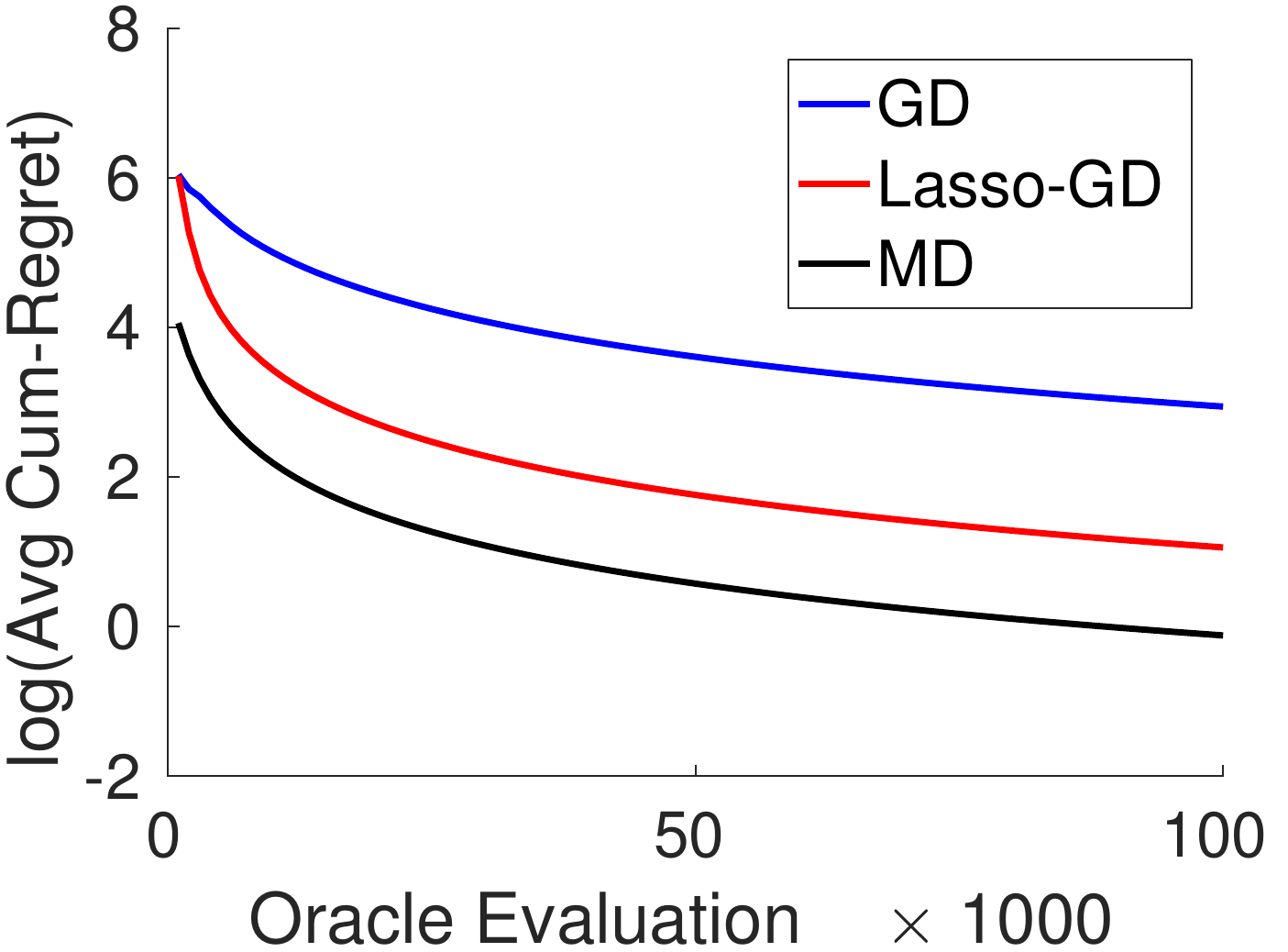}
		\caption{$s=20,d=100$}\label{fig:fourth_s20_d100}
	\end{subfigure}	
\caption{Sparse fourth-degree polynomial optimization with identity quadratic term.}\label{fig:fourth_identity}
\end{figure}

The twice de-biased estimator is, in principle, similar to the ``twicing'' trick in nonparametric kernel smoothing \citep{newey2004twicing}
that reduces estimation bias.
In particular, Corollary \ref{cor:hs-debias} shows that the $\frac{\delta}{2}\mathbb E[(z^\top H_tz)z]$ bias term is cancelled by the ``twicing'' trick,
and the remaining bias term $\widetilde\gamma$ is an order of magnitude smaller than $\gamma$ in the bias term before twicing (e.g., Lemma \ref{lem:debias}).
We also remark that the twice de-biased estimator $\widetilde g_t^{\tw}$ does \emph{not} significantly increase the computational burden,
because the method remains first-order and only (two copies of) the de-biased gradient estimate needs to be computed. 

Plugging the ``twice'' de-biased gradient estimator $\widetilde g_t^{\tw}$ into the stochastic mirror descent procedure (Algorithm \ref{alg:main})
and choosing tuning parameters $n,\lambda,\delta$ and $\eta$ appropriately, we obtain the following improved convergence rate:
\begin{theorem}
Suppose (A1) through (A4) and (A6) hold.
Suppose also that $T=\Omega(s^3\log^2 d + (1+L)^2s^2 + H^2B^2(1+L)s\log d)$ and $T\leq d$.
Let 
$\eta := Bn^{2/3}\sqrt{\frac{\log d}{T}}$, $n := \lfloor (1+L)s^{2/3}\sqrt{T}\rfloor$ and $\delta := (s\log d/n)^{1/3}$.
Then the simple regret $R_{\mathcal A}^{\simple}(T)$ can be upper bounded with probability $1-\cO(d^{-1})$ as
$$
R_{\mathcal A}^{\simple}(T) \lesssim
\widetilde\xi_{\sigma,s}B\sqrt{\log d}\left(\frac{(1+L)s^{2/3}}{T}\right)^{1/3} + \widetilde\cO(T^{-5/12}),
$$
where $\widetilde\xi_{\sigma,s} = (1+\sigma+\sigma^2/s^{2/3})$.
\label{thm:improved}
\end{theorem}
As a simple illustration consider the following example:
\begin{example}
Consider a quadratic function $f(x)=\frac{1}{2}(x-x^*)^\top Q(x-x^*)$ with (unknown) $Q\succeq 0$ being positive semi-definite and supported on $S\subseteq[d]$ with $|S|\leq s$,
meaning that $Q_{ij}=0$ if $i\notin S$ or $j\notin S$.
It is easy to verify that $f$ satisfies (A1) through (A5), and also (A6) with $L=0$ because $\nabla^2 f(x)\equiv Q$, independent of $x$.
Subsequently, applying results in Theorem \ref{thm:improved}
we obtain a convergence rate of $\cO(T^{-1/3})$ for the simple regret $R_{\mathcal A}^{\simple}(T)$.
\end{example}
More broadly, compared to Theorem \ref{thm:main}, the stochastic mirror descent algorithm with the twice de-biased gradient estimator ($\widetilde g_t^{\tw}$)
has the convergence rate of $\cO(T^{-1/3})$, which is a strict improvement over the $\cO(T^{-1/4})$ rate in Theorem \ref{thm:main}.
Such improvement is at the cost of the additional assumption of Hessian smoothness (A6);
however, the optimization algorithm remains almost unchanged and no second-order information is required at runtime.
Finally, we remark that Theorem \ref{thm:improved} only applies to the simple regret $R_{\mathcal A}^{\simple}(T)$;
we have yet to work out a similar bound for the cumulative regret 
for the particular choices of $n$ and $\delta$ in Theorem \ref{thm:improved}.

\section{SIMULATIONS}
We compare our two proposed algorithms with the baseline method for low-dimensional zeroth-order optimization (proposed in \citep{flaxman2005online})
on synthetic function examples.
We use GD to represent ``zeroth order" gradient descent algorithm proposed in~\citep{flaxman2005online}, Lasso-GD to represent Algorithm~\ref{alg:egs} and MD to represent Algorithm~\ref{alg:main}.
For our synthetic function examples, we first construct a convex low-dimensional function $f_S:\mathbb R^{|S|}\to\mathbb R$ on a uniformly chosen subset $S\subseteq[d]$ with size $s$,
and then ``extend'' $f_S$ to $f$ defined on the high-dimensional domain $\mathbb R^d$ by $f(x) \equiv f_S(x_S)$.
Functions constructed as such naturally satisfy the sparsity assumptions (A3), (A4) and (A5).
In all plots we start at the 1000th iterations (oracle evaluations) of all algorithms to avoid clutter caused by the volatile burn-in phases.
Thus, the starting points in the plots are slightly different for different algorithms.

In Figure~\ref{fig:identity} we consider sparse quadratic optimization problem with $f_S(x_S) = x_S^\top Qx_S + b^\top x_S$ where we set $Q_{ii}=1$ and $b_i=1$ for $i\in S$ and other entries to $0$.
In Figure~\ref{fig:polydecay} we consider sparse quadratic optimization problem with $f_S(x_S) = x_S^\top Qx + b^\top x_S$ where we set $Q_{ii}=i^{-\gamma}$ where $\gamma$ is the eigenvalue decay rate and $b_i=1$ for $i\in S$ and other entries to $0$.
In Figure~\ref{fig:fourth_identity} we consider sparse degree-4 polynomial optimization problem with $f_S(x) = |(x_S-b)^\top Q(x_S-b)|^2 + (x_S-b)^\top Q(x_S-b)$ where we set $Q_{ii}=1$ and $b_i=1$ for $i\in S$ and other entries to $0$.
All hyper-parameters are tuned by grid search.
The cumulative regret $R_{\mathcal A}^{\cum}(t)=\frac{1}{t}\sum_{t'=0}^{t-1}f(x_t)-f^*$ is reported for all algorithms and selected time epochs $t\leq T$.

We observe that in all our simulation settings, the vanilla gradient descent algorithm is dominated by our proposed algorithms. 
Our simulation results also suggest that the mirror descent algorithm is superior to the successive component selection algorithm.
MD is also easier to use in practice as it has fewer parameters. Thus, we recommend mirror descent algorithm for practical use.

\section{CONCLUDING REMARKS}

In this paper we consider the problem of optimizing high-dimensional functions with noisy zeroth-order oracles.
Two algorithms are proposed that work under sparsity assumptions on the gradients/Hessians or the functions themselves,
and we provide convergence bounds that only depend logarithmically on the ambient domain dimension $d$.

We view our work as a first step towards rather than the resolution of this problem.
In particular, 
in future work we hope to address the following questions:

1.~Both Algorithms \ref{alg:egs} and \ref{alg:main} require ``strong'' sparsity conditions on the input function or its gradients,
meaning that they have to be \emph{exactly} sparse.
It is an important question whether near dimension-independent convergence can be achieved with only ``weak'' sparsity assumptions, 
which only assume, for instance, that the $\ell_1$ norm of the function gradient and Hessian are bounded.
Such results, if possible, would greatly expand the applicability of the problem, 
as few functions in practice are exactly sparse.

2.~Our proposed algorithms have an $\cO(T^{-1/4})$ convergence rate, and the mirror descent algorithm converges at $\cO(T^{-1/3})$
with additional Hessian smoothness conditions.
On the other hand, in low-dimensional zeroth-order optimization it is well-understood that the optimal convergence rate is $\cO(\poly(d)T^{-1/2})$,
and there are computationally efficient algorithms achieving such rates \citep{bubeck2016kernel}.
Thus, an interesting open question is whether, under additional strong or weak sparsity conditions, 
a similar convergence rate of $\cO(\poly\log(d)T^{-1/2})$ can be achieved, 
with only poly-logarithmic dependency on the ambient dimension $d$.


\section*{ACKNOWLEDGEMENTS}

We would like to thank Renbo Zhao, Qi Lei and Yuanzhi Li for helpful discussions.
This work is supported by CMU ProSEED/BrainHub seed grant, AFRL FA8750-17-2-0212, NSF CAREER IIS1252412, and NSF DMS1713003.

\bibliographystyle{IEEE}
\bibliography{../refs}

\clearpage
\onecolumn

\appendix

\section*{APPENDIX: PROOFS}

\subsection*{Proof of Lemma \ref{lem:lasso}}

We first prove a technical lemma that bounds the $\ell_{\infty}$ norm of error vectors.
\begin{lemma}
For any $x\in\mathbb R^d$ and $z_i\in\{\pm 1\}^d$, 
with probability $1-\cO(d^{-3})$ (conditioned on $x_t$ and $z_i$)
$$
\left\|\sum_{i=1}^n{\varepsilon_i z_i}\right\|_{\infty} \lesssim \frac{\sigma}{\delta}\sqrt{\frac{\log d}{n}} + H\delta.
$$
\label{lem:varepsilon}
\end{lemma}
\begin{proof}
Let $\bar\xi_i=\xi_i/\delta\sim\mathcal N(0,\sigma^2/\delta^2)$.
Consider the following decomposition:
$$
\left\|\sum_{i=1}^n{\varepsilon_i z_i}\right\|_{\infty} \leq \frac{1}{n\delta}\left\|\sum_{i=1}^n{\bar\xi_i z_i}\right\|_{\infty} + \delta\cdot\sup_{1\leq i\leq n}\big|z_i^\top H_t(\kappa_i, z_i)z_i\big|\cdot \|z_i\|_{\infty}.
$$
The second term on the right-hand side of the above inequality is upper bounded by $\cO(H\delta)$ almost surely,
because $\|z_i\|_{\infty}\leq 1$ and $|z_i^\top H_t(\kappa_i, z_i)z_i| \leq \|H_t(\kappa_i, z_i)\|_1\|z_i\|_{\infty}^2\leq H$.
For the first term, because $\bar\xi_i$ are centered sub-Gaussian random variables independent of $z_i$ and $\|z_i\|_{\infty}\leq 1$,
we have that $1/n\cdot\|\sum_{i=1}^n{\bar\xi_i z_i}\|_{\infty}\lesssim \sqrt{\sigma^2\log d/n}$ with probability $1-\cO(d^{-3})$,
by invoking standard sub-Gaussian concentration inequalities.
\end{proof}

Now define $\widehat\theta = (\widehat g_t,\widehat\mu_t)$, $\theta_0=(g_t,\delta^{-1}f(x_t))$ and $\bar Z=(\bar z_1,\ldots,\bar z_n)$ where $\bar z_i=(z_i,1)\in\mathbb R^{d+1}$.
Define also that $Y=(\widetilde y_1,\ldots,\widetilde y_n)$.
The estimator can then be written as $\widehat\theta = \arg\min_{\theta\in\mathbb R^{d+1}}\frac{1}{n}\|\widetilde Y-\bar Z\theta\|_2^2+\lambda\|\theta\|_1$
where $\widetilde Y=\bar Z\theta_0+\varepsilon$, $\varepsilon=(\varepsilon_1,\ldots,\varepsilon_n)$.
We first establish a ``basic inequality'' type results that are essential in performance analysis of Lasso type estimators.
By optimality of $\widehat\theta$, we have that
$$
\frac{1}{n}\|Y-\bar Z\widehat\theta\|_2^2+\lambda\|\widehat\theta\|_1\leq \frac{1}{n}\|Y-\bar Z\theta_0\|_2^2+\lambda\|\theta_0\|_1 = \frac{1}{n}\|\varepsilon\|_2^2+\lambda\|\theta_0\|_1.
$$
Re-organizing terms we obtain
$$
\lambda\|\widehat\theta\|_1 \leq \lambda\|\theta_0\|_1 + \frac{2}{n}(\widehat\theta-\theta_0)^\top \bar Z^\top\varepsilon.
$$
On the other hand, by H\"{o}lder's inequality and Lemma \ref{lem:varepsilon} we have, with probability $1-\cO(d^{-2})$,
$$
\frac{2}{n}(\widehat\theta-\theta_0)^\top \bar Z^\top\varepsilon \leq 2\|\widehat\theta-\theta_0\|_1\cdot \left\|\frac{1}{n}\bar Z^\top\varepsilon\right\|_{\infty}
\lesssim \|\widehat\theta-\theta_0\|_{1}\cdot \left(\frac{\sigma}{\delta}\sqrt{\frac{\log d}{n}} + H\delta\right).
$$
Subsequently, if $\lambda\leq c_0(\sigma\delta^{-1}\sqrt{\log d/n}+H\delta)$ for some sufficiently small $c_0>0$,
we have that $\|\widehat\theta\|_1\leq \|\theta_0\|_1 + 1/2\|\widehat\theta-\theta_0\|_1$.
Multiplying by 2 and adding $\|\widehat\theta-\theta_0\|_1$ on both sides of the inequality we obtain
$\|\widehat\theta-\theta_0\|_1\leq 2(\|\widehat\theta-\theta_0\|_1+\|\widehat\theta_0\|_1-\|\widehat\theta\|_1)$.
Recall that $\theta_0$ is sparse and let $\bar S=S\cup\{d+1\}$ be the support of $\theta_0$.
We then have $\|(\widehat\theta-\theta_0)_{\bar S^c}+\|(\theta_0)_{\bar S^c}\|_1-\|\widehat\theta_{\bar S^c}\|_1=0$
and hence $\|(\widehat\theta-\theta_0)_{\bar S^c}\|_1-\|(\widehat\theta-\theta_0)_{\bar S}\|_1 \leq \|\widehat\theta-\theta_0\|_1 \leq 2\|(\widehat\theta-\theta_0)_{\bar S}\|_1$.
Thus,
\begin{equation}
\|(\widehat\theta-\theta_0)_{\bar S^c}\|_1 \leq 3\|(\widehat\theta-\theta_0)_{\bar S}\|_1.
\label{eq:basic-ineq}
\end{equation}

Now consider $\widehat\theta$ that minimizes $\frac{1}{n}\|Y-\bar Z\theta\|_2^2+\lambda\|\theta\|_1$.
By KKT condition we have that
$$
\left\|\frac{1}{n}\bar Z^\top(Y-\bar Z\widehat\theta)\right\|_{\infty} \leq \frac{\lambda}{2}.
$$
Define $\widehat\Sigma=\frac{1}{n}\bar Z^\top\bar Z$ and recall that $Y=\bar Z\theta_0+\varepsilon$.
Invoking Lemma \ref{lem:varepsilon} and the scaling of $\lambda$ we have that, with probability $1-\cO(d^{-2})$
\begin{equation}
\|\widehat\Sigma(\widehat\theta-\theta_0)\|_{\infty} \leq \frac{\lambda}{2} + \left\|\frac{1}{n}\bar Z^\top\varepsilon\right\| \lesssim \frac{\sigma}{\delta}\sqrt{\frac{\log d}{n}} + \delta H.
\label{eq:kkt}
\end{equation}
By definition of $\{\bar z_i\}_{i=1}^n$, we know that $\widehat\Sigma_{jj}=1$ for all $j=1,\ldots,d+1$ and $\mathbb E[\widehat\Sigma_{jk}]=0$ for $j\neq k$.
By Hoeffding's inequality \citep{hoeffding1963probability} and union bound we have that with probability $1-\cO(d^{-2})$, $\|\widehat\Sigma-I_{(d+1)\times(d+1)}\|_{\infty}\lesssim \sqrt{\log d/n}$, where $\|\cdot\|_{\infty}$ denotes the maximum absolute value of matrix entries.
Also note that $\widehat\theta-\theta_0$ satisfies $\|(\widehat\theta-\theta_0)_{\bar S^c}\|_1\leq 3\|(\widehat\theta-\theta_0)_{\bar S}\|_1$ thanks to Eq.~(\ref{eq:basic-ineq}).
Subsequently, 
\begin{align}
\|\widehat\theta-\theta_0\|_{\infty}
&\leq \|\widehat\Sigma(\widehat\theta-\theta_0)\|_{\infty} + \|(\widehat\Sigma-I)(\widehat\theta-\theta_0)\|_{\infty}\nonumber\\
&\leq \|\widehat\Sigma(\widehat\theta-\theta_0)\|_{\infty} + \|\widehat\Sigma-I\|_{\infty}\|\widehat\theta-\theta_0\|_1\nonumber\\
&\leq \|\widehat\Sigma(\widehat\theta-\theta_0)\|_{\infty} + \|\widehat\Sigma-I\|_{\infty}\cdot 4\|(\widehat\theta-\theta_0)_{\bar S}\|_1\nonumber\\
&\leq \|\widehat\Sigma(\widehat\theta-\theta_0)\|_{\infty} + \|\widehat\Sigma-I\|_{\infty}\cdot 4(s+1)\|\widehat\theta-\theta_0\|_{\infty}\nonumber\\
&\lesssim \frac{\sigma}{\delta}\sqrt{\frac{\log d}{n}} + \delta H + \sqrt{\frac{s^2\log d}{n}}\cdot \|\widehat\theta-\theta_0\|_{\infty}.
\label{eq:kkt-decompose}
\end{align}
Combining Eq.~(\ref{eq:kkt-decompose}) together with the scaling $n=\Omega(s^2\log d)$ we complete the proof of Lemma \ref{lem:lasso}.
Note that the statement on the $\ell_1$ error $\|\widehat\theta-\theta_0\|_1$ is a simple consequence of the basic inequality Eq.~(\ref{eq:basic-ineq}).

\subsection*{Proof of Theorem \ref{thm:model-selection-rate}}

The basis of our algorithm is the analysis of the finite-difference algorithm proposed by \cite{flaxman2005online} under low dimensions.
In particular, applying the analysis in \citep{agarwal2010optimal} for low-dimensional strongly smooth functions, we have for every epoch $t<s$
\begin{equation*}
\mathbb E[f(x_t)] - \inf_{x\in\widetilde{\mathcal X},x_{\widehat S_t^c}=0}f(x) \lesssim \poly(s,\sigma, H, \|x^*_{\widehat S_t}\|_1)\cdot T^{-1/3},
\end{equation*}
where $x_t$ is the solution point at the $t$th epoch in Algorithm \ref{alg:egs} 
and $\poly(\cdot)$ is any polynomial function of constant degrees.
Recall that $\|x^*_{\widehat S_t}\|_1 \leq \|x^*\|_1 \leq B$ by Assumption (A2). 
Using Markov's inequality we have that with probability 0.9,
\begin{equation}
f(x_t) - \inf_{x\in\widetilde{\mathcal X},x_{\widehat S_t^c}=0}f(x) \lesssim \poly(s,\sigma, H, \|x^*_{\widehat S_t}\|_1)\cdot T^{-1/3}, \;\;\;\;\;\forall t=0,\ldots,s.
\label{eq:lowdim-regret}
\end{equation}

We are now ready to prove Theorem \ref{thm:model-selection-rate}.
Let $\widehat S=\widehat S_t$ be the subset when Algorithm \ref{alg:egs} terminates.
In the rest of the proof we assume the conclusions in Corollary \ref{cor:model-selection} and Lemma \ref{lem:lasso} hold, which happens with probability $1-\cO(d^{-1})$.
Define $\Delta S=S\backslash\widehat S$, 
$x^* := \inf_{x\in\mathcal X}f(x)$ and $x_t^*=\inf_{x\in\widetilde{\mathcal X}, x_{\widehat S_t^c}=0}f(x)$.
Assumption (A5) implies that $x^*$ can be chosen such that $x^*_{S^c}=0$.
Also, if $\Delta_S=\emptyset$ we know that $x_t^*=x^*$ and Theorem \ref{thm:model-selection-rate} automatically holds due to Eq.~(\ref{eq:lowdim-regret}).
Therefore in the rest of the proof we shall assume that $\Delta_S\neq\emptyset$.

Because $\Delta_S\neq\emptyset$ and $|S|=s$, we must have $|\widehat S_t|<s$.
From the description of Algorithm \ref{alg:egs}, it can only happen with $\widehat S_t=\widehat S_{t-1}$.
We then have that
\begin{align}
f(x_{T+1})-f(x^*)
&= f(x_{t-1}^*)-f(x^*) + f(\widehat x_{t-1})-f(x_{t-1}^*)\nonumber\\
&\leq f(x_{t-1}^*)-f(x^*) + \poly(s,\sigma, H,\|x^*\|_1)\cdot T^{-1/3}\label{eq:markov}\\
&\leq \nabla f(x_{t-1}^*)^\top(x_{t-1}^*-x^*) + \poly(s,\sigma, H,\|x^*\|_1)\cdot T^{-1/3},\label{eq:ms-intermediate-1}
\end{align}
where Eq.~(\ref{eq:markov}) holds with probability at least 0.9, thanks to Eq.~(\ref{eq:lowdim-regret}).
Because $x_{t-1}^*$ is the minimizer of $f$ on vectors in $\widetilde{\mathcal X}$ that are supported on $\widehat S=\widehat S_{t-1}=\widehat S_t$,
and that both $x_{t-1}^*$ and $x^*$ truncated on $\widehat S$ are feasible (i.e., in the restrained set $\widetilde{\mathcal X}$),
it must hold that $\langle [\nabla f(x_{t-1}^*)]_{\widehat S}, (x_{t-1}^*-x^*)_{\widehat S}\rangle\leq 0$
by first-order optimality conditions.
On the other hand, by Corollary \ref{cor:model-selection} and the definition of $\widehat S_t$,
we have that $\|[\nabla f(x_{t-1}^*)_{\Delta_S}]\|_{\infty} \leq 2\eta$.
Also note that $(x^*-x_{t-1}^*)_{S^c}=0$ and $[x_{t-1}^*]_{\Delta_S}=0$. Subsequently, 
\begin{equation}
\nabla f(x_{t-1}^*)^\top(x_{t-1}^*-x^*)
\leq \big|\langle\nabla f(x_{t-1}^*)_{\Delta_S}, x^*_{\Delta_S}\rangle\big|
\leq \|[\nabla f(x_{t-1}^*)]_{\Delta_S}\|_{\infty}\|x^*_{\Delta_S}\|_1
\leq 2\eta\|x^*\|_1.
\label{eq:ms-intermediate-2}
\end{equation}
Combining Eqs.~(\ref{eq:ms-intermediate-1},\ref{eq:ms-intermediate-2}) and the scalings of $\eta,\delta,\lambda$ and $T'=T/2s$ we complete the proof of Theorem \ref{thm:model-selection-rate}.

\subsection*{Proof of Lemma \ref{lem:debias}}

We use the ``full-length'' parameterization $\widetilde\theta_t=\widehat\theta_t+\frac{1}{n}\bar Z_t^\top(\widetilde Y_t-\bar Z_t\widehat\theta_t)$,
where $\widehat\theta_t,\bar Z_t$ and $\widetilde Y_t$ are notations defined in the proof of Lemma \ref{lem:lasso} (with subscripts $t$ added to emphasize that
both $Z_t$ and $\widetilde Y_t$ are specific to the $t$th epoch in Algorithm \ref{alg:main}).
Because $\widetilde Y_t=\bar Z_t\theta_{0t}+\varepsilon_t$ (where $\theta_{0t}=\nabla f(x_t)$ and $\varepsilon=(\varepsilon_{t1},\ldots,\varepsilon_{tn})$, with $\varepsilon_{ti}$ defined in Eq.~(\ref{eq:gradient-model})).
we have
\begin{align*}
\widetilde\theta_t &= \widehat\theta_t + \frac{1}{n}\bar Z_t^\top(\bar Z_t\theta_{0t}+\varepsilon_t-\bar Z_t\widehat\theta_t)
= \theta_{0t} + \frac{1}{n}\bar Z_t^\top\varepsilon_t + (\widehat\Sigma- I_{(d+1)\times (d+1)})(\widehat\theta_t-\theta_{0t}),
\end{align*}
where $\widehat\Sigma = \frac{1}{n}\bar Z_t^\top\bar Z_t$.
Recall that $\varepsilon_{ti} = \xi_i/\delta + \delta z_i^\top H_t(\kappa_i,z_i)z_i$.
Define $b_i=z_i^\top H_t(\kappa_i,z_i)z_i$ and $b=(b_1,\ldots,b_n)$.
Also note that the first $d$ components of $\widetilde\theta_t$ are identical to $\widetilde g_t$ defined in Eq.~(\ref{eq:debias}).
Subsequently, 
\begin{equation}
\widehat g_t = g_t + \underbrace{\frac{1}{n\delta}Z_t^\top\xi}_{:=\zeta_t} + \underbrace{\frac{\delta}{n}Z_t^\top b + \left[(\widehat\Sigma- I_{(d+1)\times (d+1)})(\widehat\theta_t-\theta_{0t})\right]_{1:d}}_{:=\gamma_t}.
\label{eq:debias-intermediate-1}
\end{equation}

In Eq.~(\ref{eq:debias-intermediate-1}) we divide $\widehat g_t-g_t$ into two terms.
We first consider the term $\zeta_t := \frac{1}{n\delta}Z_t^\top\xi$.
It is clear that $\mathbb E[\zeta_t|x_t]=0$ because $\mathbb E[\xi|x_t,Z_t]=0$.
Now consider any $d$-dimensional vector $a\in\mathbb R^d$, and to simplify notations all derivations below are conditioned on $x_t$.
For any $i\in[n]$, $z_{ti}^\top a$ are i.i.d.~sub-Gaussian random variables with common parameter $\nu^2=\|a\|_2^2$.
Also, $\bar\xi_i$ is a sub-Gaussian random variable with parameter $\sigma^2$ and is independent of $z_{ti}^\top a$.
Thus, invoking Lemma \ref{lem:product-subgaussian} we have that $\xi_i z_{ti}^\top a$ is a sub-exponential random variable
with parameters $\nu=\alpha/\sqrt{2}\lesssim \sigma\|a\|_2$.
Consequently, $\langle\zeta_t,a\rangle=\frac{1}{n\delta}\sum_{i=1}^n{\xi_i z_{ti}^\top a}$ is a centered sub-exponential random variable
with parameters $\nu=\sqrt{n/2}\cdot\alpha\lesssim \sigma\|a\|_2/\delta\sqrt{n}$.

We next consider the term $\gamma_t = \frac{\delta}{n}Z_t^\top b+(\widehat\Sigma-I)(\widehat\theta_t-\theta_{0t})$.
By Assumption (A3) we know that $\|b\|_{\infty}\leq \delta H$.
Subsequently, by H\"{o}lder's inequality we have that
\begin{align*}
\|\gamma_t\|_{\infty} 
&\leq \frac{\delta}{n}\|Z_t\|_{1,\infty}\|b\|_{\infty} + \|\widehat\Sigma-I\|_{\infty}\|\widehat\theta_t-\theta_{t0}\|_1\\
&\lesssim H\delta + \sqrt{\frac{\log d}{n}}\left(\frac{\sigma s}{\delta}\sqrt{\frac{\log d}{n}} + s\delta H\right).
\end{align*}
where the second inequality holds with probability $1-\cO(d^{-2})$ thanks to Lemma \ref{lem:lasso}.

\subsection*{Proof of Theorem \ref{thm:main}}

We first note that the cumulative regret $R^{\cum}_{\mathcal A}(T)$ can be upper bounded as
$$
R^{\cum}_{\mathcal A}(T) \lesssim \left[\frac{1}{T'}\sum_{t=0}^{T'-1}f(x_{t})-f^*\right] + \sup_{t}\sup_{z\in \{\pm 1\}^d} \big|f(x_t+\delta z)-f(x_t)\big|. 
$$
Because $\|\nabla f(x)\|_1\leq H$ for all $x\in\mathcal X$ and $z\in\{\pm 1\}^d$, using H\"{o}lder's inequality we have that
$$
\big|f(x_t+\delta z)-f(x_t)\big| \leq \delta H \lesssim B\left(\frac{s\log^2 d}{T}\right)^{1/4},
$$
which is a second-order term.
Thus, to prove upper bounds on $R^{\cum}_{\mathcal A}(T)$ it suffices to consider only $\frac{1}{T'}\sum_{t=0}^{T'-1}f(x_t) - f^*$.

We next cite the result in \citep{lan2012optimal} that gives explicit cumulative regret bounds for mirror descent with approximate gradients:
\begin{lemma}[\cite{lan2012optimal}, Lemma 3]
Let $\|\cdot\|_{\psi}$ and $\|\cdot\|_{\psi^*}$ be a pair of conjugate norms, and let $\Delta_\psi(\cdot,\cdot)$ be a Bregman divergence that is $\kappa$-strongly convex
with respect to $\|\cdot\|_\psi$.
Suppose $f$ is $\widetilde H$-smooth with respect to $\|\cdot\|_\psi$, meaning that
$f(y)\leq f(x)+\nabla f(x)^\top(y-x) + \frac{\widetilde H}{2}\|x-y\|_\psi^2$ for all $x,y\in\mathcal X$,
and $\eta < \kappa/\widetilde H$.
Define $g_t=\nabla f(x_t)$, and let $x_0,\ldots,x_{T'-1}$ be iterations in Algorithm \ref{alg:main}.
Then for every $0\leq t\leq T'-1$ and any $x^*\in{\widetilde{\mathcal X}}$,
\begin{equation}
\eta\left[f(x_{t+1})-f(x^*)\right] + \Delta_\psi(x_{t+1},x^*)
\leq \Delta_\psi(x_{t},x^*) + \eta\langle\widetilde g_t-g_t,x^*-x_t\rangle + \frac{\eta^2\|\widetilde g_t-g_t\|_{\psi^*}^2}{2(\kappa-\widetilde H\eta)}.
\label{eq:stmd-rate}
\end{equation}
\label{lem:stmd-rate}
\end{lemma}

Adding both sides of Eq.~(\ref{eq:stmd-rate}) from $t=0$ to $t=T'-1$, telescoping and noting that $\Delta_\psi(x_{T'},x^*)\geq 0$, we obtain
\begin{align}
\frac{1}{T'}\sum_{t=0}^{T'-1} f(x_t)-f(x^*)
&\leq \frac{\Delta_\psi(x_0,x^*)}{\eta T'} + \frac{1}{T'}\sum_{t=0}^{T'-1}{\langle\widetilde g_t-g_t,x_t-x^*\rangle} + \frac{\eta}{2(\kappa-H\eta)}\cdot\sup_{0\leq t<T'}\|\widetilde g_t-g_t\|_{\psi^*}^2.
\label{eq:telescope}
\end{align}

Set $\|\cdot\|_\psi=\|\cdot\|_a$ for $a=\frac{2\log d}{2\log d-1}$.
It is easy to verify that under Assumption (A3), the function $f$ satisfies
\begin{align*}
f(y) &\geq f(x) + \nabla f(x)^\top (y-x) + H\|y-x\|_\infty^2\\
&\geq f(x)+\nabla f(x)^\top(y-x)+\widetilde H\|y-x\|_\psi^2
\end{align*}
for all $x,y\in\mathcal X$ with $\widetilde H \leq {e}H$,
because $\|x-y\|_1^2\leq d^{2(1-1/a)}\|x-y\|_a^2 \leq d^{1/\log d}\|x-y\|_1^2=e\|x-y\|_1^2$ by H\"{o}lder's inequality.
In addition, by definition of Bregman divergence we have that
\begin{equation}
\Delta_\psi(x_0,x^*) \leq \frac{1}{2(a-1)}\|x^*\|_a^2 \leq \frac{1}{2(a-1)}\|x^*\|_1^2 \leq \|x^*\|_1^2\log d \leq B^2\log d,
\label{eq:init-bregman}
\end{equation}
where the first inequality holds because $\psi_a(x_0)=\psi_a(0)=0$ and $\nabla\psi_a(x_0)=\nabla\psi_a(0)=0$ for $a>1$.

We next upper bound the $\frac{1}{T'}\sum_{t=0}^{T'-1}\langle\widetilde g_t-g_t,x^*-x_t\rangle$ and $\|\widetilde g_t-g_t\|_{\psi^*}^2$ terms.
By Lemma \ref{lem:debias} and sub-exponential concentration inequalities (e.g., Lemma \ref{lem:bernstein}), we have that
with probability $1-\cO(d^{-1})$
\begin{align*}
\|\widetilde g_t-g_t\|_{\infty} 
&\leq  \|\zeta_t\|_{\infty} + \|\gamma_t\|_{\infty}
\lesssim \frac{\sigma}{\delta}\left(\sqrt{\frac{\log d}{n}}+\frac{\log d}{n}\right) + H\delta + \frac{\sigma s\log d}{\delta n}
\lesssim \frac{\sigma}{\delta}\sqrt{\frac{\log d}{n}} + H\delta
\end{align*}
uniformly over all $t'\in\{0,\ldots,T'-1\}$,
where the last inequality holds because $n=\Omega(s^2\log d)$.
Subsequently, by H\"{o}lder's inequality we have that
\begin{equation}
\sup_{0\leq t<T'}\|\widetilde g_t-g_t\|_{\psi^*}^2 \leq d^{2(a-1)/a}\cdot \sup_{0\leq t<T'}\|\widetilde g_t-g_t\|_{\infty}^2 
\lesssim  \frac{\sigma^2\log d}{\delta^2n}  + H^2\delta^2.
\label{eq:term2}
\end{equation}

We now consider the first term $\frac{1}{T'}\sum_{t=0}^{T'-1}\langle\widetilde g_t-g_t,x^*-x_t\rangle \leq \frac{1}{T'}\sum_{t=0}^{T'-1}{X_t} + \sup_{0\leq t\leq T'-1}\|\gamma_t\|_{\infty}\|x^*-x_t\|_1$,
where $X_t := \langle \zeta_t,x^*-x_t\rangle$. 
By Lemma \ref{lem:debias}, we know that $X_t|X_1,\ldots,X_{t-1}$ is a centered sub-exponential random variable
with parameters $\nu=\sqrt{n/2}\cdot\alpha \lesssim \sigma\|x^*-x_t\|_2/\delta\sqrt{n} \lesssim \sigma\|x^*\|_1/\delta\sqrt{n}$.
Invoking concentration inequalities for sub-exponential martingales (\citep{victor1999general}, also phrased as Lemma \ref{lem:martingale-bernstein} 
for a simplified version in the appendix) and the definition that $T'=T/n$, we have with probability $1-\cO(d^{-1})$
$$
\bigg|\frac{1}{T'}\sum_{t=0}^{T'-1}{\langle\zeta_t,x^*-x_t\rangle}\bigg| \lesssim \frac{\sigma\|x^*\|_1}{\delta}\left(\sqrt{\frac{\log d}{T}} + \frac{\log d}{T}\right) \lesssim \frac{\sigma\|x^*\|_1}{\delta}\sqrt{\frac{\log d}{T}},
$$
where the last inequality holds because $T\geq n=\Omega(s^2\log d)$.
Thus, 
\begin{equation}
\bigg|\frac{1}{T'}\sum_{t=0}^{T'-1}{\langle\widetilde g_t-g_t,x^*-x_t\rangle} \bigg|
\lesssim  \frac{\sigma\|x^*\|_1}{\delta}\sqrt{\frac{\log d}{T}} + \|x^*\|_1\left(H\delta + \frac{\sigma s\log d}{\delta n}\right).
\label{eq:term1}
\end{equation}
Combining Eqs.~(\ref{eq:init-bregman},\ref{eq:term2},\ref{eq:term1}) with Eq.~(\ref{eq:telescope})
and taking $x^*$ to be a minimizer of $f$ on $\mathcal X$ that satisfies $\|x^*\|_1\leq B$, 
 we obtain
\begin{align}
\frac{1}{T'}\sum_{t=0}^{T'-1}{f(x_t)-f(x^*)}
&\lesssim \frac{\|x^*\|_1^2\log d}{\eta}\frac{n}{T} + \frac{\sigma\|x^*\|_1}{\delta}\sqrt{\frac{\log d}{T}}+\|x^*\|_1\left(H\delta+\frac{\sigma s\log d}{\delta n}\right) +\eta\left(\frac{\sigma^2\log d}{\delta^2n}  + H^2\delta^2\right)\nonumber\\
&\leq \frac{B^2\log d}{\eta}\frac{n}{T} + \frac{\sigma B}{\delta}\sqrt{\frac{\log d}{T}}+B\left(H\delta + \frac{\sigma s\log d}{\delta n}\right)+\eta\left(\frac{\sigma^2\log d}{\delta^2n}  + H^2\delta^2\right)
\label{eq:md-final}
\end{align}
with probability $1-\cO(d^{-1})$, provided that $\eta<\kappa/2H=1/2H$.

We are now ready to prove Theorem \ref{thm:main}.
By the conditions we impose on $T$ and the choices of $\eta$ and $n$, 
it is easy to verify that $\eta< 1/2H$, $n=\Omega(s^2\log d)$ and $n = \cO(T)$.
Subsequently,
\begin{align*}
\frac{1}{T'}\sum_{t=0}^{T'-1}{f(x_t)-f(x^*)}
&\lesssim B\sqrt{\frac{n\log d}{T}} + \sigma B\sqrt{\frac{n}{sT}} + B(\sigma+H)\sqrt{\frac{s\log d}{n}} + B\sqrt{\frac{n\log d}{T}}\left(\frac{\sigma^2}{s} + 
\widetilde\cO(n^{-1})\right)\\
&\lesssim B\left(\frac{(1+H)^2s\log^2 d}{T}\right)^{1/4} + \frac{\sigma B\sqrt{(1+H)}}{s^{1/4}T^{1/4}}
+ \frac{B(\sigma+H)}{\sqrt{1+H}}\left(\frac{s\log^2 d}{T}\right)^{1/4} \\
&+ B\left(\frac{(1+H)^2s\log d}{T}\right)^{1/4}\left(\frac{\sigma^2}{s} + \widetilde\cO(T^{-1/2})\right)\\
&\lesssim \left(B\sqrt{\log d} + \frac{\sigma B}{\sqrt{s}} + \frac{\sigma^2 B}{s}\right)\left[\frac{(1+H)^2s}{T}\right]^{1/4} + B(\sigma+\sqrt{H})\sqrt{\log d}\left[\frac{s}{T}\right]^{1/4} +\widetilde \cO(T^{-1/2})\\
&\lesssim (1+\sigma+\sigma^2/s)B\sqrt{\log d}\left[\frac{(1+H)^2s}{T}\right]^{1/4} +\widetilde \cO(T^{-1/2}).
\end{align*}

\subsection*{Proof of Lemma \ref{lem:hs-debias}}

Using the model Eq.~(\ref{eq:gradient-model}) we can decompose $\widetilde g_t(\delta)-g_t$ as
\begin{align*}
\widetilde g_t(\delta)-g_t &= \frac{\delta}{2}\mathbb E\left[(z^\top H_t z)z\right] + \underbrace{\frac{1}{n\delta}Z_t^\top\xi}_{:=\widetilde\zeta_t(\delta)} + \underbrace{\frac{\delta}{2n}\sum_{i=1}^n{(z_i^\top H_tz_i)z_i-\mathbb E[(z^\top H_tz)z]}}_{:=\widetilde\beta_t(\delta)} \\
&+ \underbrace{\frac{\delta}{2n}\sum_{i=1}^n{(z_i^\top(H_t(\delta z_i)-H_t)z_i)z_i} + \left[(\widehat\Sigma-I)(\widehat\theta_t-\theta_{0t})\right]_{1:d}}_{:=\widetilde\gamma_t(\delta)},
\end{align*}
where $\widehat\Sigma$, $\widehat\theta_t$ and $\theta_{0t}$ are similarly defined as in the proof of Lemma \ref{lem:debias}.
The sub-exponentiality of $\langle\widetilde\zeta_t(\delta),a\rangle$ for any $a\in\mathbb R^d$ is established in Lemma \ref{lem:debias}.
We next consider $\widetilde\beta_t(\delta)$.
For any $a\in\mathbb R^d$ consider $\langle\widetilde\beta_t(\delta),a\rangle = \frac{\delta}{2n}\sum_{i=1}^n{X_i(a)}$
where $X_i(a)=(z_i^\top H_tz_i)(z_i^\top a)-\mathbb E[(z_i^\top H_tz_i)(z_i^\top a)]$ are centered i.i.d.~random variables conditioned on $H_t$ and $x_t$.
In addition, $|X_i(a)| \leq 2\|H_t\|_{1}\|z_i\|_{\infty}^2\cdot \|a\|_1\|z_i\|_{\infty} \lesssim  H\|a\|_1$ almost surely.
Therefore, $X_i(a)$ is a sub-Gaussian random variable with parameter $\nu= H\|a\|_1$,
and hence $\langle\widetilde \beta_t(\delta),a\rangle$ is a sub-Gaussian random variable with parameter $\nu=\delta  H\|a\|_1/\sqrt{n}$.
Finally, for the deterministic term $\widetilde\gamma_t(\delta)$, we have that
\begin{align*}
\|\widetilde\gamma_t(\delta)\|_{\infty}
&\leq \frac{\delta}{2}\sup_{z\in\{\pm 1\}^d} \|H_t(\delta z)-H_t\|_{1}\|z\|_\infty^2 + \|(\widehat\Sigma-I)(\widehat\theta_t-\theta_{0t})\|_{\infty}\\
&\leq \frac{\delta}{2}\sup_{z\in\{\pm 1\}^d} L\cdot \|\delta z\|_{\infty}\|z\|_{\infty}^2 + \|\widehat\Sigma-I\|_{\max}\|\widehat\theta_t-\theta_{0t}\|_{\infty}\\
&\lesssim L\delta^2 + \sqrt{\frac{\log d}{n}}\left(\frac{\sigma s}{\delta}\sqrt{\frac{\log d}{n}} + s \delta H\right)\\
&\lesssim L\delta^2 + \frac{\sigma s\log d}{n\delta} + s\delta H\sqrt{\frac{\log d}{n}}.
\end{align*}

\subsection*{Proof of Theorem \ref{thm:improved}}

Because $f$ is convex, $R_{\mathcal A}^{\simple}(T) = f(x_{T+1})-f^* \leq \frac{1}{T'}\sum_{t=0}^{T'-1}{f(x_t)-f^*}$.
Thus it suffices to upper bound $\frac{1}{T'}\sum_{t=0}^{T'-1}f(x_t)-f(x^*)$, where $x^*\in\mathcal X$, $\|x^*\|_1\leq B$ is a minimizer of $f$ over $\mathcal X$.
Using the strategy in the proof of Theorem \ref{thm:main}, this amounts to upper bound (with high probability) $\|\widetilde g_t^{\tw}-g_t\|_{\psi^*}^2$ and 
$\frac{1}{T'}\sum_{t=0}^{T'-1}\langle \widetilde g_t^{\tw}-g_t,x^*-x_t\rangle$.

For the first term, using sub-exponentiality of $\widetilde\zeta_t$ and sub-gaussianity of $\widetilde\beta_t$, we have with probability $1-\cO(d^{-1})$ 
uniformly over all $t\in\{0,\ldots,T'-1\}$, 
\begin{align*}
\|\widetilde g_t^{\tw}-g_t\|_{\infty} 
&\leq \|\widetilde\zeta_t\|_{\infty} + \|\widetilde\beta_t\|_{\infty} + \|\widetilde\gamma_t\|_{\infty}\\
&\lesssim \frac{\sigma}{\delta}\left(\sqrt{\frac{\log d}{n}} + \frac{\log d}{n}\right) + \delta H\sqrt{\frac{\log d}{n}} + L\delta^2 + H\delta\sqrt{\frac{s^2\log d}{n}} + \frac{\sigma s\log d}{\delta n}\\
&\lesssim  \left(\frac{\sigma}{\delta} + s\delta H\right)\sqrt{\frac{\log d}{n}} +L\delta^2,
\end{align*}
where the last inequality holds because $n=\Omega(s^2\log d)$. Subsequently, with probability $1-\cO(d^{-1})$
\begin{equation}
\sup_{0\leq t\leq T'-1}\|\widetilde g_t^{\tw}-g_t\|_{\psi^*}^2 \lesssim \left(\frac{\sigma^2}{\delta^2} + s^2\delta^2 H^2\right){\frac{\log d}{n}} +L^2\delta^4.
\label{eq:improve-intermediate-1}
\end{equation}

For the other term $\frac{1}{T'}\sum_{t=0}^{T'-1}\langle\widetilde g_t^{\tw}-g_t,x^*-x_t\rangle$, again using concentration inequalities of sub-exponential/sub-Gaussian martingales 
and noting that $\|x^*-x_t\|_2 \leq \|x^*-x_t\|_1\leq 2B$, 
we have
\begin{align}
\frac{1}{T'}\sum_{t=0}^{T'-1}\langle\widetilde g_t^{\tw}-g_t, x^*-x_t\rangle
&= \frac{1}{T'}\sum_{t=0}^{T'-1}\langle \widetilde\zeta_t+\widetilde\beta_t+\widetilde\gamma_t, x^*-x_t\rangle\nonumber \\
&\lesssim \left(\frac{\sigma}{\delta} + s\delta H\right) B\sqrt{\frac{\log d}{T}} + B\left(L\delta^2 + \frac{\sigma s\log d}{\delta n} + s\delta H\sqrt{\frac{\log d}{n}}\right).
\label{eq:improve-intermediate-2}
\end{align}
Subsequently, combining Eqs.~(\ref{eq:improve-intermediate-1},\ref{eq:improve-intermediate-2}) with Eq.~(\ref{eq:telescope}) we have
\begin{align}
\frac{1}{T'}\sum_{t=0}^{T'-1}f(x_t)-f(x^*)
&\lesssim \frac{B^2\log d}{\eta}\frac{n}{T} +  \left(\frac{\sigma}{\delta} + s\delta H\right) B\sqrt{\frac{\log d}{T}} 
+ (B+\eta)\left(L\delta^2 + \frac{\sigma s\log d}{\delta n} + s\delta H\sqrt{\frac{\log d}{n}}\right) \nonumber\\
&+ \eta\left(\frac{\sigma^2}{\delta^2} + s^2\delta^2 H^2\right){\frac{\log d}{n}} +\eta L^2\delta^4.
\label{eq:main-hs}
\end{align}

We are now ready to prove Theorem \ref{thm:improved}.
It is easy to verify that with the condition imposed on $T$ and the selection of $\eta$ and $n$, it holds that
$\eta<1/2H$, $n=\Omega(s^2\log d)$ and $n\leq T/10$.
Subsequently,
\begin{align*}
&\frac{1}{T'}\sum_{t=0}^{T'-1}f(x_t)-f(x^*)\\
&\lesssim  Bn^{1/3}\sqrt{\frac{\log d}{T}} + \left[\sigma\left(\frac{n}{s\log d}\right)^{1/3} + \widetilde\cO(n^{-1/3})\right]B\sqrt{\frac{\log d}{T}}
+ \left(B+\widetilde\cO\left(\frac{n^{2/3}}{\sqrt{T}}\right)\right)\left[(L+\sigma)\left(\frac{s\log d}{n}\right)^{2/3} +\widetilde \cO(n^{-5/6}) \right]\\
&+ Bn^{2/3}\sqrt{\frac{\log d}{T}}\left(\sigma^2\left(\frac{n}{s\log d}\right)^{2/3} +\widetilde \cO(n^{-2/3})\right)\frac{\log d}{n} 
+ Bn^{2/3}\sqrt{\frac{\log d}{T}}L^2\left(\frac{s\log d}{n}\right)^{4/3}\\
&\lesssim Bn^{1/3}\sqrt{\frac{\log d}{T}} + \sigma B\left(\frac{n}{s\log d}\right)^{1/3}\sqrt{\frac{\log d}{T}} + B(L+\sigma)\left(\frac{s\log d}{n}\right)^{2/3}
+ \sigma^2 B\left(\frac{n}{s^2\log^2 d}\right)^{1/3}\sqrt{\frac{\log d}{T}} + \widetilde\cO(T^{-5/12})\\
&\lesssim \left(B\sqrt{\log d} + \frac{\sigma B\sqrt{\log d}}{s^{1/3}} + \frac{\sigma^2 B\sqrt{\log d}}{s^{2/3}}\right)\left[\frac{(1+L)s^{2/3}}{T}\right]^{1/3} 
+ \frac{B(L+\sigma)}{(1+L)^{2/3}}\left(\frac{s^{2/3}\log d}{T}\right)^{1/3} + \widetilde\cO(T^{-5/12})\\
&\lesssim \left(B\sqrt{\log d} + \frac{\sigma B\sqrt{\log d}}{s^{1/3}} + \frac{\sigma^2 B\sqrt{\log d}}{s^{2/3}}\right)\left[\frac{(1+L)s^{2/3}}{T}\right]^{1/3} 
+ B\sigma\sqrt{\log d}\left(\frac{(1+L)s^{2/3}}{T}\right)^{1/3} + \widetilde\cO(T^{-5/12})\\
&\lesssim (1+\sigma+\sigma^2/s^{2/3})B\sqrt{\log d}\left(\frac{(1+L)s^{2/3}}{T}\right)^{1/3} + \widetilde\cO(T^{-5/12}).
\end{align*}

\subsection*{Additional tail inequalities}

\begin{lemma}
Suppose $X$ and $Y$ are centered sub-Gaussian random variables with parameters $\nu_1^2$ and $\nu_2^2$, respectively.
Then $XY$ is a centered sub-exponential random variable with parameter $\nu=\sqrt{2}v$ and $\alpha=2v$, where $v=2e^{2/e+1}\nu_1\nu_2$.
\label{lem:product-subgaussian}
\end{lemma}
\begin{proof}
$XY$ is clearly centered because $\mathbb EXY=\mathbb EX\cdot\mathbb EY=0$, thanks to independence.
We next bound $\mathbb E[|XY|^k]$ for $k\geq 3$ (i.e., verification of the Bernstein's condition).
Because $X$ and $Y$ are independent, we have that $\mathbb E[|XY|^k] = \mathbb E|X|^k\cdot \mathbb E|Y|^k$.
In addition, because $X$ is a centered sub-Gaussian random variable with parameter $\nu_1^2$, it holds that 
$(\mathbb E|X|^k)^{1/k} \leq \nu_1 e^{1/e}\sqrt{k}$.
Similarly, $(\mathbb E|X|^k)^{1/k}\leq \nu_2 e^{1/e}\sqrt{k}$.
Subsequently,
\begin{align*}
\mathbb E|XY|^k 
\leq \left(e^{2/e}\nu_1\nu_2\right)^k\cdot k^k
\leq \left(e^{2/e}\nu_1\nu_2\right)^k\cdot e^k k! \leq \frac{1}{2}k!\cdot \left(2e^{2/e+1}\nu_1\nu_2\right)^k.
\end{align*}
where in the second inequality we use the Stirling's approximation inequality that $\sqrt{2\pi k}k^ke^{-k} \leq k!$.
The sub-exponential parameter of $XY$ can then be determined.
\end{proof}

\begin{lemma}[Bernstein's inequality]
Suppose $X$ is a sub-exponential random variable with parameters $\nu$ and $\alpha$. 
$$
\Pr\left[\big|X-\mathbb EX\big|>t\right] \leq \left\{\begin{array}{ll}
2\exp\left\{-t^2/2\nu^2\right\}, & 0<t\leq \nu^2/\alpha;\\
2\exp\left\{-t/2\alpha\right\},& t>\nu^2/\alpha.\end{array}\right.
$$
\label{lem:bernstein}
\end{lemma}

The following lemma is a simplified version of Theorem 1.2A in \citep{victor1999general} (note that the original form in \citep{victor1999general} is
one-sided; the two-sided version below can be trivially obtained by considering $-X_1,\ldots,-X_n$ and applying the union bound).
\begin{lemma}[Bernstein's inequality for martingales]
Suppose $X_1,\ldots,X_n$ are random variables such that $\mathbb E[X_j|X_1,\ldots,X_{j-1}]=0$ 
and $\mathbb E[X_j^2|X_1,\ldots,X_{j-1}]\leq \sigma^2$ for all $t=1,\ldots,n$.
Further assume that $\mathbb E[|X_j|^k|X_1,\ldots,X_{j-1}] \leq \frac{1}{2}k!\sigma^2b^{k-2}$ for all integers $k\geq 3$.
Then for all $t>0$, 
$$
\Pr\left[\bigg|\sum_{j=1}^n{X_j}\bigg| \geq t\right] \leq 2\exp\left\{-\frac{t^2}{2(n\sigma^2+bt)}\right\}.
$$
\label{lem:martingale-bernstein}
\end{lemma}

\end{document}